\documentclass[dvipsnames]{article} 
\usepackage{iclr2023_conference,times}

\usepackage{amsmath,amsfonts,bm}

\def\eqref#1{equation~\ref{#1}}

\def\1{\bm{1}}

\DeclareMathAlphabet{\mathsfit}{\encodingdefault}{\sfdefault}{m}{sl}
\SetMathAlphabet{\mathsfit}{bold}{\encodingdefault}{\sfdefault}{bx}{n}

\usepackage{url}
\usepackage{booktabs}
\usepackage{diagbox}
\usepackage{graphicx} 
\usepackage{amsthm}
\usepackage{multirow}
\usepackage{epigraph}
\usepackage{colortbl}
\usepackage{wrapfig}
\usepackage{caption}
\definecolor{lightgreen}{rgb}{0.9, 0.9, 0.9}

\usepackage[colorlinks, linkcolor=RoyalBlue, anchorcolor=BrickRed, citecolor=RoyalBlue, urlcolor=RoyalBlue]{hyperref}

\usepackage{cleveref}
\usepackage{hyperref}
\usepackage{fontawesome}
\crefname{section}{§}{§§}
\Crefname{section}{§}{§§}

\newcommand\refsec[1]{Section~\hyperref[sec:#1]{\ref{sec:#1}}}
\newcommand\refsecs[2]{\hyperref[sec:#1]{§\ref{sec:#1}:~\textsc{#1}}, \hyperref[sec:#2]{§\ref{sec:#2}:~\textsc{#2}}}

\newcommand{\cmt}[1]{{#1}}

\title{Predicting Emergent Abilities with Infinite Resolution Evaluation}

\author{Shengding Hu$^1$, Xin Liu$^{2}$, Xu Han$^{1,3}$\footnotemark[1], Xinrong Zhang$^{1}$, Chaoqun He$^{1}$, Weilin Zhao$^{1}$, \\ \textbf{Yankai Lin$^{4}$}, \textbf{Ning Ding$^{1}$,} \textbf{Zebin Ou$^{5}$,} \textbf{Guoyang Zeng$^{6}$,} \textbf{Zhiyuan Liu$^{1}$\thanks{Corresponding Authors.} ,} \textbf{Maosong Sun$^{1}$\footnotemark[1]}  \\
$^1$Department of Computer Science and Technology, Tsinghua University \quad 
\\$^2$Beijing Language and Culture University.\\
$^3$Shanghai Artificial Intelligence Laboratory \\
$^4$Renmin University of China. \quad 
$^5$Zhihu Inc. \quad 
$^6$Modelbest Inc.\\
\texttt{hsd23@mails.tsinghua.edu.cn} \\
}
\iclrfinalcopy 
\begin{document}

\maketitle

\begin{abstract}
\looseness=-1 The scientific scale-up of large language models (LLMs) necessitates a comprehensive understanding of their scaling properties. However, the existing literature on the scaling properties only yields an incomplete answer: optimization loss decreases predictably as the model size increases, in line with established scaling law; yet no scaling law for task has been established and the task performances are far from predictable during scaling. Task performances typically show minor gains on small models until they improve dramatically once models exceed a size threshold, exemplifying the ``emergent abilities''. In this study, we discover that small models, although they exhibit minor performance, demonstrate critical and consistent task performance improvements that are not captured by conventional evaluation strategies due to insufficient measurement resolution. To measure such improvements, we introduce \textsc{PassUntil}, an evaluation strategy with theoretically infinite resolution, through massive sampling in the decoding phase. With \textsc{PassUntil}, we conduct a quantitative investigation into the scaling law of task performance. The investigation contains two parts. Firstly, a strict \textit{task scaling law} that is not conventionally known to exist, is identified, enhancing the predictability of task performances. Remarkably, we are able to predict the performance of the 2.4B model on code generation with merely 0.05\% deviation before training starts, which is the first systematic attempt to verify predictable scaling proposed by GPT-4's report~\citep{openai2023gpt4}. Secondly, underpinned by \textsc{PassUntil}, \cmt{we are able to study emergent abilities quantitatively. We identify a kind of \textbf{accelerated emergence} whose scaling curve cannot be fitted by standard scaling law function and has a increasing speed. We then examine two hypothesis and imply that the ``multiple circuits hypothesis'' might be responsible for the accelerated emergence.}
\end{abstract}

\vspace{-0.6cm}
\hspace{235pt}\parbox[b]{0.3\textwidth}
{
\epigraph{\textit{``See the world in a grain of sand''}}
}
\vspace{-1.2cm}

\section{Introduction}

\looseness=-1 Large Language Models (LLMs)~\citep{devlin2018bert, raffel2020exploring, brown2020language, chowdhery2022palm} have become a center of interest among AI researchers recently. These models, trained on expansive datasets and furnished with an enormous number of parameters, have demonstrated unparalleled proficiency across diverse domains, such as text generation~\citep{dubois2023alpacafarm}, code completion~\citep{chen2021evaluating, roziere2023code}, and academic test~\citep{hendrycks2020measuring}.

\looseness=-1 The impressive success of these LLMs depends heavily on scaling up the model parameters and pre-training data volume. It has been consistently observed that, when considering a continuum of models with nearly identical architectures, larger models coupled with increased pre-training corpora consistently yield diminished training loss. This observation has been mathematically formalized as the scaling law of loss~\citep{kaplan2020scaling, henighan2020scaling}, which states that the reducible loss achieved by the model in the log scale is linear to the model size in the log scale. Scaling law has provided guidance for the scientific scaling of LLMs, including determining the balance of the model size and pre-training data size~\citep{hoffmann2022training, muennighoff2023scaling}. This has transformed what was once a somewhat blind scaling process into a methodology underpinned by empirical assurance. 
Nonetheless, such beneficial scaling law yield predictions solely on the loss, not extending to the real task performance encountered in practice. This divergence establishes a substantial gap in a comprehensive scaling-up methodology~\citep{ganguli2022predictability}.

\vspace{-0.2cm}
\begin{figure}[!htbp]
        \centering
        \includegraphics[width=\linewidth]{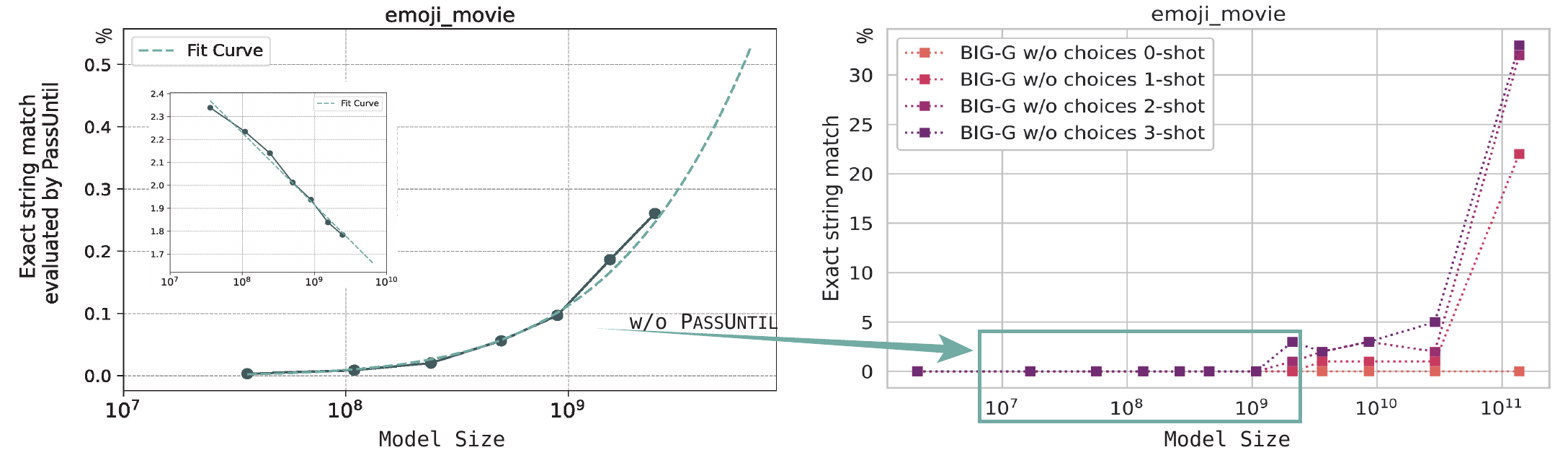}
     \vspace{-0.8cm}
    \caption{We can discriminate subtle performance improvement (left), which is evaluated as all zeros in conventional methods (right). The right figure directly uses Figure 9(a) in~\cite{sorscher2022beyond} as a comparison, which the authors utilize to illustrate a ``break-through'' behavior in task performance. The internal figure inside the left figure shows the performances in a $\log(-\log(\cdot))$ space, which displays strong linearity, supporting the task scaling law (Eq.(\ref{eq:task_scaling_raw})).}
    \label{fig:themefig}
\end{figure}

The challenge in extending loss caling law to task performance predominantly stems from the \textit{discontinuity} observed in task performance during scaling. Language models below a certain size yield trivial performance, i.e., random guessing on multiple choices or zero scores on generation tasks. However, when the model size surpasses a certain threshold, a distinct surge in performance appears, which leads to substantially non-trivial performance.  This phenomenon is summarized as the ``emergent abilities''~\citep{srivastava2022beyond, wei2022emergent}, and is observed across various model families and tasks. It seems that qualitative changes happen inside the model, which makes the model start to manifest unique capabilities.
While these emerging phenomenon indicate that LLMs are becoming stronger, they complicate the prediction on task performance.

\looseness=-1 A pivotal question arises: \textbf{can we unlock predictable scaling of the task performance, from the apparent discontinuities?} We hypothesize that the perceived discontinuity from trivial to excellent performance might stem from limited evaluation resolution\footnote{By ``resolution'', we view evaluation as a measurement of the real probability of completing a task. And resolution is the smallest probability difference that the evaluation strategy can detect.}. By employing a more nuanced resolution, one could potentially uncover the scaling law for tasks. \cmt{The most related work to ours is \cite{schaeffer2023emergent}, which proposes two methodology to make emergent abilities continuous, i.e., ``change of metrics'' and ``increase resolution'' by expanding test set size. Our motivation diverges from the ``change of metric'' approach of \cite{schaeffer2023emergent}, which posits that employing other continuous metrics can cause emergent abilities to disappear. A limitation of alternative smooth metrics (e.g., distribution distance) is they yield insufficient insights into the target metrics (e.g., exact match) that evaluators intuitively perceive. In contrast, our method extends the ``increase resolution'' approach in a novel way, which target directly at predicting the performance such as  code generation in our experiments.}

\looseness=-1  We introduce an evaluation strategy named \textsc{PassUntil} that, for the first time, enables quantitative exploration of the scaling properties of task performance. \textsc{PassUntil} deploys extensive random sampling in the decoding phase (e.g., $10^5$ sampling times), and evaluates each sampling result \textit{until} any generation \textit{passes} the target test. Therefore, this evaluation strategy has infinite measurement resolution as long as computational resources are not bounded. Moreover, it can provide maximum likelihood estimates of target metrics such as accuracy and exact match. To refine our evaluation resolution and accuracy, we suggest fitting to instance-level scaling law since different test instances might have different speeds of performance improvement during scaling.

With the proposed evaluation strategy, we delve into the scaling law governing task performance. 
To begin with, we train two series of models ranging from 0.03B to 2.4B. These models strictly adhere to pre-training loss scaling law, providing a solid foundation for analyzing task performance scaling behavior. We mainly disclose two findings in our exploration.

\looseness=-1 Firstly, task performances are predictable with \textsc{PassUntil}. We validate the presence of subtle but non-negligible performance in smaller models that can be captured by \textsc{PassUntil}. These performances are on the order of $10^{-5}$ and exhibit steady enhancement as the model scales up. Subsequently, we derive the mathematical form of \textbf{task scaling law}, experimentally verifying an almost strict linear relationship between \(\log(-\log(\textsc{PU}))\) and \(\log(N)\), where $\textsc{PU}$ denotes the estimation of target metric given by \textsc{PassUntil} and $N$ is the number of model parameters. 
This relationship enables us to attain highly accurate predictions. For instance, in the code generation task, our predictions exhibit a mere 0.05\% deviation from the actual values.

\cmt{Secondly, we discover a phenomenon of \textbf{accelerated emergence}. To begin with, we discover that the shape of the task scaling curve is not uniform across tasks. Several task manifest scaling functions that diverge from the typical task scaling law. In other words, their scaling curve is smooth and incremental but can not be fitted by the typical scaling law function. Their scaling curve of \(\log(-\log(\textsc{PU}))\) w.r.t. \(\log(N)\) is concave,  which is akin to an acceleration in the performance scaling speed. We provide a mathematical definition of such phenomenon. With the quantitative definition,} we exclude a possible  multi-step reasoning explanation~\citep{schaeffer2023emergent}, and propose an alternative hypothesis. This hypothesis is predicated on potential transformer circuits~\citep{elhage2021mathematical} that are used to explain the ``grokking'' phenomenon~\citep{power2022grokking, varma2023explaining}. It is in harmony with the observed scaling function.

Our work represents the first open-source attempt regarding the predictability of task performance. While GPT-4's report~\citep{openai2023gpt4} has initiated this exploration, it has not provided comprehensive details. We will open-source all checkpoints to facilitate future research in this direction.

\vspace{-0.3cm}
\section{Related Work}

Predicting task performance before training is an aspirational objective for the development of predictable AI systems, and a multitude of studies approach this aim from various perspectives.

\textbf{Loss Scaling Law.} Scaling phenomena have been observed across a broad spectrum of deep learning architectures. The power-law scaling behavior of loss in RNN-based models is investigated in~\citet{hestness2017deep}. ~\citet{kaplan2020scaling} delineate the loss scaling trends for Transformer-based language models and explores the scaling behavior of optimal hyper-parameters. They formally established the following scaling law
\begin{equation}
\label{eq:loss_scaling_law}
    L = c N^{-\alpha} + L_0,
\end{equation}
where $N$ is the number of non-embedding parameters of LLM, $c, \alpha$ are positive coefficients, and $L_0$ is the irreducible loss representing the randomness in data. This formulation has catalyzed the proliferation of LLMs. Subsequently, scaling laws are established for various domains and scenarios, including multi-modality~\citep{henighan2020scaling, zhai2022scaling}, computation constraint scenario~\citep{hoffmann2022training}, data engineering~\citep{muennighoff2023scaling,sorscher2022beyond}, and reinforcement learning~\citep{gao2023scaling}. ~\cite{yao2023research} extend the scaling law into loss prediction by introducing hyper-parameter scaling methods. The relationship of our work with these existing literature is twofold. First, these works concentrate on training and validation loss metrics, which do not reliably predict task performance. Second, our research builds on these scaling laws and extends the mathematical form of Eq.(\ref{eq:loss_scaling_law}) to the scaling law of task performance.

\looseness=-1 \textbf{Scaling Behavior of Task Performance.} 
Despite the predictable decrement in LLM loss, task performance improvements are twisted during scaling. While some tasks, predominantly those relying on memorization of knowledge, have shown progressive improvement, numerous tasks exhibit breakthrough behavior as model size increases~\citep{srivastava2022beyond,wei2022emergent}.  ~\citet{wei2022emergent} illustrate that the concept of ``emergence'' is also pertinent to prompting techniques such as Chain-of-Thought~\citep{wei2022chain} and In-context Learning~\citep{brown2020language}, complicating the pursuit of understanding the scaling law of task performance. It appears that the law of loss scaling offers no assurance for task performance, engendering a lack of guidance in pre-training methodology.  
Fortunately, several studies endeavor to demystify these emergent abilities. GPT-4's technical report~\citep{openai2023gpt4} reports that GPT-4's task performance can be predicted with less than $1/10000$ of computation, albeit without disclosing the methodology and acknowledging that certain abilities are still beyond prediction. Subsequent research~\citep{schaeffer2023emergent} attributes emergence to two reasons. The first one is non-smooth metrics. We disagree with it since the alternative metrics could not explain the sudden increase in target metrics such as exact match, which are of paramount interest to us. We align with their second attribution to improve resolution by adding more test samples. Different from their method, we propose a practical method to improve resolution without the need of adding test samples. Our work is also the first open-source attempt to quantitatively investigate the scaling behavior of task performance, proposing task scaling law and accelerated emergence phenomenon.

\vspace{-0.2cm}
\section{Pilot Experiments on Increasing Random Sample Numbers}
\label{sec:pilot}
We initiate our exploration by visualizing the effect of improving evaluation resolution on open-sourced models.
We choose four small models and evaluate them on two subsets of BigBench task~\citep{srivastava2022beyond}: Emoji Movie and Date Understanding (see \hypertarget{back:c_4_2}{Appendix}~\ref{app:emoji_movies} \hypertarget{back:c_4_3}{and}~\ref{app:dateunderstanding} for the subsets). We employ beam search and random sampling (with three sample times: 1, 100, and 10,000) during decoding. If any sampled answer of a test instance is evaluated as correct, then the instance is marked as ``passed''. We present the number of passed instances in Figure~\ref{tab:opensourcemodel_rs}.

\begin{figure}[htbp]
\centering
\scalebox{0.95}{
    \includegraphics[width=0.95\textwidth]{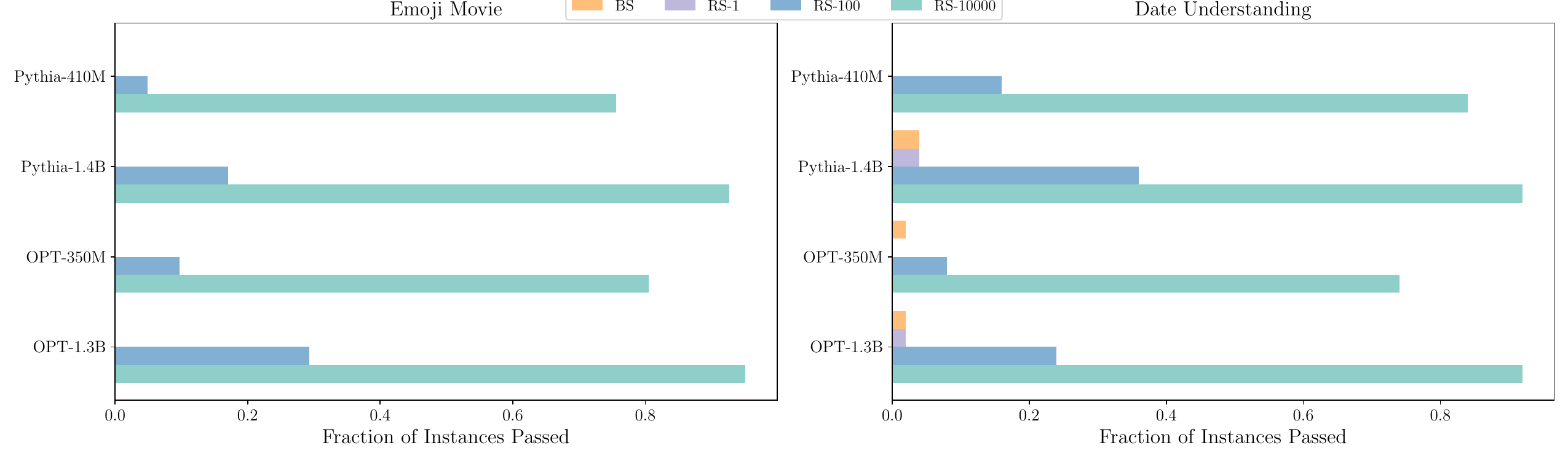}
}
    \caption{\cmt{BS denotes beam search, RS-$K$ denotes random sampling $K$ times.}}
    \label{tab:opensourcemodel_rs}
\vspace{-0.3cm}
\end{figure}

We can see that even for such tasks presenting substantial difficulty to small models, most instances are passable with enough random sampling times, which will contribute to the subtle task performance improvement. Inspired by this observation, we propose our evaluation strategy that centered around improving the resolution of evaluation.

\section{Methods}
 In this section, we describe our methods to increase the resolution of evaluation, which empowers the investigation of the scaling behavior of task performance. The first is an evaluation strategy \textsc{PassUntil}, and the second is an instance-level scaling curve fit. We also derive the task scaling law based on the loss scaling law.

\subsection{Infinite Resolution with \textsc{PassUntil}}
We view task performance evaluation as the measurement of the probability of a model passing~\footnote{The definition of ``pass'' does not need to be generating exactly the ground truth answer. For example, suppose we predict model's performance on AlpacaEval~\citep{alpaca_eval}, we can define ``pass'' as the model generation being better than GPT-4, judged by GPT-4. Therefore the ``pass'' has broad application.} a task.  Given a task instance $s$, suppose the probability that a model pass it is $P(s)$, our job is to estimate $\mathbb{E}_s[P(s)]$. Randomly sampling a fixed time $K$ could estimate $P(s)$. However, it is hard to define the budget $K$ that is both acceptable in computation and has enough resolution for hard samples that have small $P(s)$. We propose \textsc{PassUntil}, which performs an evaluation right after an answer is generated and determines whether it is passed before we sample the next generation. We stop sampling until $r$ (a constant) samples have passed the evaluation and record the sampling number $K$. We name the estimate of $P(s)$ as the \textsc{PassUntil} score \textsc{PU}, \cmt{which is defined as
\begin{equation}
\label{equ:pudef}
    \textsc{PU} = \frac{r}{K}
\end{equation}
Theoretically, $\textsc{PU}$ has the capability to measure success rates that are infinitesimally small. The \textsc{PassUntil} has the following properties.}

\newtheorem{theorem}{Theorem}
\cmt{
\begin{theorem}
$\textsc{PU}$ is a maximum likelihood estimate for $P(s)$.
\end{theorem}
\vspace{-0.3cm}
\begin{proof}
The failure time $f = K - r$ follows the negative binomial distribution with success probability $P(s)$. $r/K$ is known to be an maximum likelihood estimate for $P(s)$. 
\end{proof}
}

 In practice, we set $r$ to as small as $1$ or $2$ considering the efficiency of evaluation.  We also set the upper bound of $K$ to a large number, such as $10^5$, to prevent endless sampling if we encounder an extremely low $P(s)$. Note that many instances stop before reaching this upper-bound. Next we discuss the necessity and limitations of \textsc{PassUntil}. 

\cmt{
\looseness=-1 \textbf{Necessity.} }Generally,  deriving $P(s)$ theoretically from the token probability on the ground truth solution is not feasible. This is due to two primary facts: firstly, there are likely to be multiple viable solutions; secondly, even though there is only one solution, there exist multiple decoding approaches besides the optimal tokenization to decode the solution\footnote{For example, [4513], [717,18], and [16,17,18] all decode into string ``123'' in GPT-4's tokenizer with vocab ``cl100k-base''.}.  

\cmt{
\looseness=-1 \textbf{Limitations.}} (1) Currently, our evaluation strategy is designed to be applicable when a random baseline achieves $P(s) = 0$. In the context of multiple-choice grade as the evaluation metric, evaluations tend to exhibit a biased high score relative to the true performance of the model (e.g., $P(s) = 0.25$ with random guess for four options). This random noise can overshadow the improvements made by smaller models. The exploration of scaling law for tasks with non-zero random baselines remains a subject for future research. 
(2) We currently only consider random sampling as a viable target decoding strategy due to its widespread use in LLMs. Using beam search as target decoding strategies and their relationship with random sampling poses an interesting avenue for future exploration and study.

\subsection{From Loss-Scaling Law to Task Scaling Law}
Then, we derive the task scaling law that \textsc{PassUntil} will follow. We assume that the test loss of generating the next token decreases according to the scaling law of Eq.(\ref{eq:loss_scaling_law}). 
\begin{equation}
\label{eq:task_scaling_raw}
    \textsc{PU} \sim \prod_{i=1}^{|y|} P(y_i|x_{1:|x|}, y_{1:i-1}) = \prod_{i=1}^{|y|} \exp(-{c_i}{{N}^{-\alpha_i}} - L_{0i}),
\end{equation}
where $x_{1:|x|}$ is the input sequence and $y_{1:|y|}$ is the most probable sequence that decodes the correct answer (assuming its dominance compared to other sequences).
Assume that the test sample is passable given a sufficiently potent LLM, then the irreducible loss for each token $L_{0i}$ approaches $0$. And assume the test loss of each token in the answer is decreasing with uniform speed when scaling (i.e., $a_i=a,\forall i$), we can derive the following function for $\textsc{PU} $ on task performance:
\begin{equation}
\label{eq:exp_scaling}
    \textsc{PU}(c, \alpha; N) \sim \operatorname{exp}(\sum_i -{c_i}{{N}^{-\alpha}} ) = \operatorname{exp}(-{c}{{N}^{-\alpha}})
\end{equation}
where $c=\sum_i c_i$. The resulting mathematical model is similar to that in GPT-4 technical report~\citep{openai2023gpt4} and Equation (4) in ~\cite{schaeffer2023emergent}. 

\subsection{Fitting Strategy}
\textbf{Dataset-level Fit.} When fitting the parameters  $c, \alpha$ in \textsc{PU}, a dataset-level fit is plausible. For the $j$-th model in the scaling curve, the individual test sample's \textsc{PU} is first averaged over the test set to procure $\operatorname{log}(-\operatorname{log}(\textsc{PU}(N_j))$, followed by a linear regression to $\operatorname{log}N_j$.

\textbf{Instance-level Fit.}
\label{sec:instance_level_pass}
We notice that differences between instances lead to different scaling behaviors, which means a dataset-level fit might not be accurate when the difficulty in the test set is diverse. For example, \textsc{PU} on easy
questions get saturated to 1 on a small model while the hard questions still receive trivial performance (see \hypertarget{back:a_2}{Appendix}~\ref{app:instance_level_fit_illustration} for illustration). We propose to fit an individual \textsc{PassUntil} score (\textsc{IPU}) for each question and aggregate them into an estimate for the whole dataset. 
\begin{equation}
    {\textsc{PU}}(\{c_s, a_s\}; N) = \frac{1}{|S|}\sum_s {\operatorname{IPU}}(c_s, a_s; N)
\end{equation}
\section{Predictable Scaling Experiments}

In this section, we demonstrate how the proposed framework works in practice. We first pre-train two series of language models ranging from $0.03$B to $2.4$B using two dataset mixtures. We predict the performance of the $2.4$B model based on the performance of the rest of the models in the series. 

\subsection{Scaling Configurations.}
\label{sec:train}
\textbf{Model Configurations.} We propose to keep a consistent ``shape'' of the Transformers while expanding their sizes. For the $i$-th model in the scaling curve, we set the number of layers to be $4i$, the number of attention heads to be $\lfloor \frac{i(8+i)}{4} \rfloor$, and the dimension of head to be $64$. This results in the hidden state's dimension $d_m$ being $d_hn_h$. We set the dimension of the feed-forward layer to be $2.5d_m$. The specific values are listed in the model configurations in Table~\ref{tab:model_configs} of Appendix~\ref{app:modelconf}.
The architecture is similar to LLaMA~\citep{touvron2023llama} (see \hypertarget{back:c_1}{Appendix}~\ref{app:modelconf} for details).

\textbf{Pre-training Corpora.} For series 1, we use the StarCoder dataset~\citep{li2023starcoder} as our pre-training data. For series 2, we use a mixture of StarCoder and Pile~\citep{gao2020pile} dataset. Leveraging the optimal compute LLMs~\citep{hoffmann2022training},  we set the maximum pre-training tokens for each model size to be the $20N$, where $N$ is the number of non-embedding parameters of the model. The detailed portion within the data mixture can be seen in  \hypertarget{back:c_2}{Appendix}~\ref{app:datamixture}. 

\vspace{-0.3cm}
\begin{figure}[!htbp]
        \centering
        \includegraphics[width=0.85\linewidth]{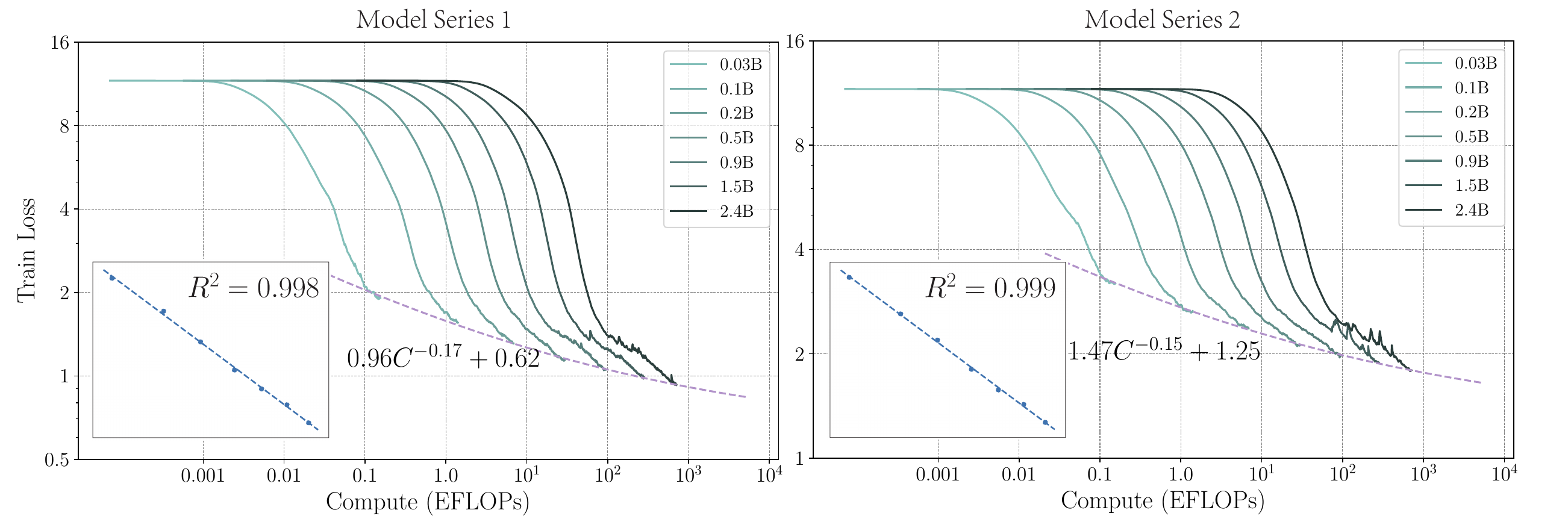}
        \label{fig:train_loss_2series}
\vspace{-0.3cm}
\caption{Training loss of the two series of models trained on different data mixtures. The internal figure illustrates the end-step reducible loss relative to model size, represented in logarithmic scale.}
    \label{fig:loss_scaling}
\end{figure}

\textbf{Hyper-parameters.} 
\hypertarget{back:c_hyper} Hyper-parameters are also of paramount importance in training a series of models that scale successfully. We examine the cosine learning rate scheduler, aligning our approach with that of \cite{hoffmann2022training}, and determine the critical batch size in accordance with \cite{kaplan2020scaling}. Nonetheless, due to constraints in space, we move the details to Appendix~\ref{app:hyperparameters}.

\subsection{Loss Scaling Law Verification.}
\looseness=-1 We present the training loss curves for models in Figure~\ref{fig:loss_scaling}. It is evident that the end-step training losses decrease in line with the scaling law. These empirically observed loss scaling laws lay a foundation for the subsequent approximation of task performance. Note that despite the occurrence of the loss spike in the 1.5B and 2.4B models, convergence to the scaling law is ultimately achieved, exemplifying the robustness of such an empirical law.

\subsection{Dataset-level Fit}
We select HumanEval~\citep{chen2021evaluating}, Emoji Movie, and Date Understanding~\citep{srivastava2022beyond} as the evaluation tasks.
Note that Emoji Movie is conventionally cited as representing “emergent abilities”~\citep{srivastava2022beyond} (see the right figure in Figure~\ref{fig:themefig}).
HumanEval is assessed using a zero-shot learning setting, while Emoji Movie and Date Understanding are evaluated employing 4-shot In-context Learning~\citep{brown2020language}. We additionally use Chain-of-Thought Reasoning~\citep{wei2022chain} for Emoji Movie. See \hypertarget{back:c_4}{Appendix}~\ref{app:testset} for the illustration and evaluation details of each task. We remove the distracting test instances from our evaluation list. For Emoji Movie, we remove the movie names that are common words (e.g., ``\textit{it}'') identified by NLTK~\citep{bird2009natural}. These common words make the exact string match susceptible to random guess's correctness (
See \hypertarget{back:d_5} Appendix~\ref{app:removing_distracting}  for details).

\begin{figure}[!htbp]
        \centering
        \includegraphics[width=0.87\linewidth]{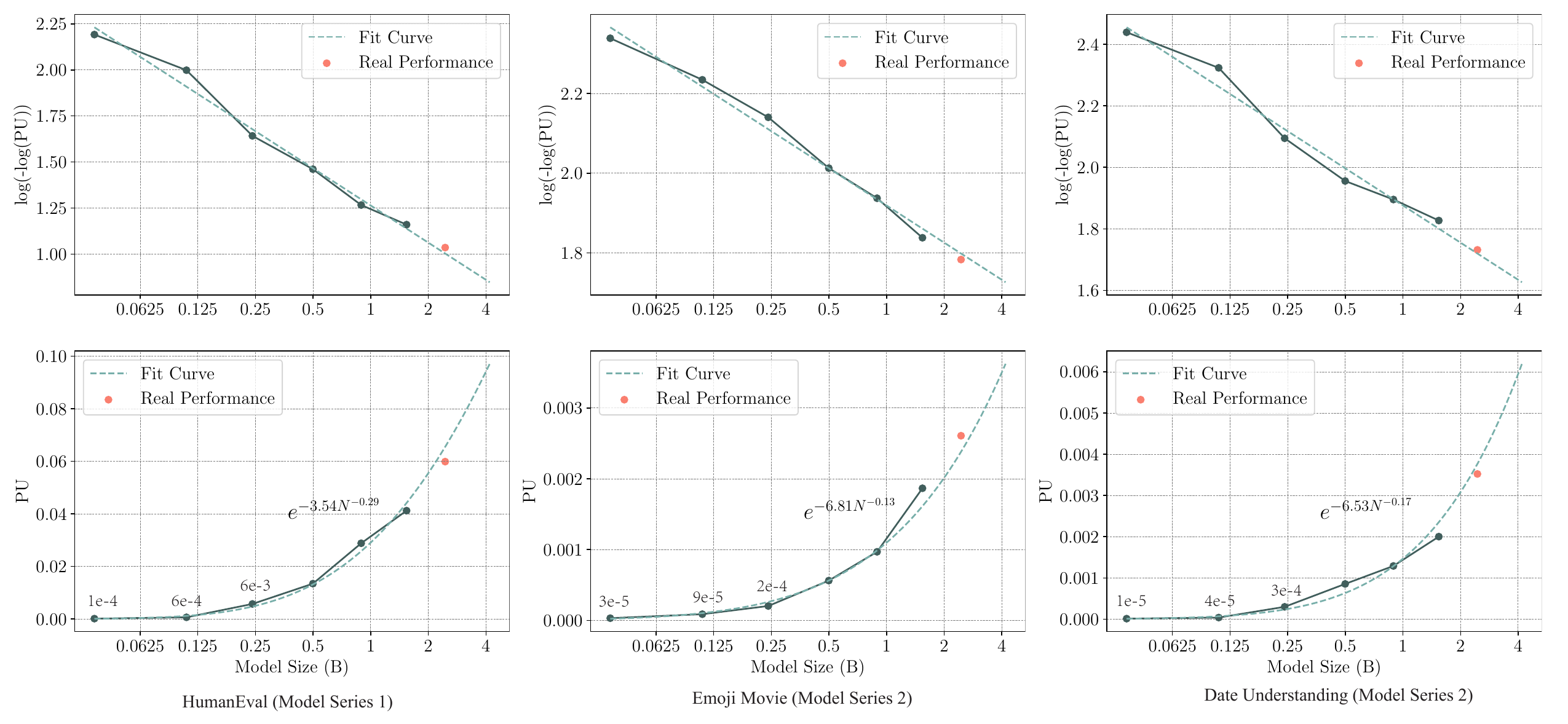}
     \vspace{-0.3cm}
    \caption{Task performance scales predictably with model scale. The {\color[rgb]{0.85,0.25,0.25}red} points denote the real performance of 2.4B model, which are close to the task scaling laws fitted from 0.03B to 1.5B.}
    \label{fig:dpu}
\vspace{-0.7cm}
\end{figure}

We observe that all three tasks exhibit a strong linear relationship between $\log(-\log(\textsc{PU}))$ and $\log(N)$, verifying the success of task scaling law given by Eq.(\ref{eq:task_scaling_raw}). The estimation of the scaling law functions utilizes the 0.03b to 1.5B models, which predicts the performance of the 2.4B model with small yet acceptable deviations. 

\vspace{-0.3cm}
\subsection{Instance-level Fit}
\vspace{-0.2cm}
According to \cref{sec:instance_level_pass}, we take the difference among test samples into consideration to improve the estimation. We plot how instance-level \textsc{PassUntil} scales in Figure~\ref{fig:idp_with_modelsize_part1} of \hypertarget{back:d_4}{Appendix}~\ref{app:otheripu_sample}. 
The fitted curves demonstrate that the performances of different instances not only originate from unique starting points but also scale at varying speeds.
Nevertheless, they can be fitted by task scaling law individually. Some instances deviate from the scaling law, which needs future investigation.

\begin{figure}[!htbp]
\vspace{-0.3cm}
    \centering
    \begin{minipage}[b]{1\textwidth}
     \centering
\scalebox{0.9}{
\begin{tabular}{lcccc}
    \toprule
  \textbf{Method}  & {\textbf{HumanEval (1)}} & {\textbf{HumanEval (2)}} & {\textbf{Date Understanding (2)}} & {\textbf{Emoji Movie (2)}} \\
    \midrule
    Real Value & 0.05990 & 0.04279 & 0.00346 & 0.002608 \\
    \hline
    Dataset-level Fit & 0.06550 & 0.05191 & 0.00377 & \textbf{0.002381}\\
    Instance-level Fit & \textbf{0.05987} & \textbf{0.04402} &\textbf{ 0.00352} & 0.003112\\
    \bottomrule
\end{tabular}
}
        \captionof{table}{Prediction of our framework compared to the real performance on two series of models. The number after the task denotes the model series used in the evaluation.}
    \label{tab:final_fit_result}
    \end{minipage}
    
    \begin{minipage}[b]{0.34\textwidth}
        \centering
        \includegraphics[scale=0.2]{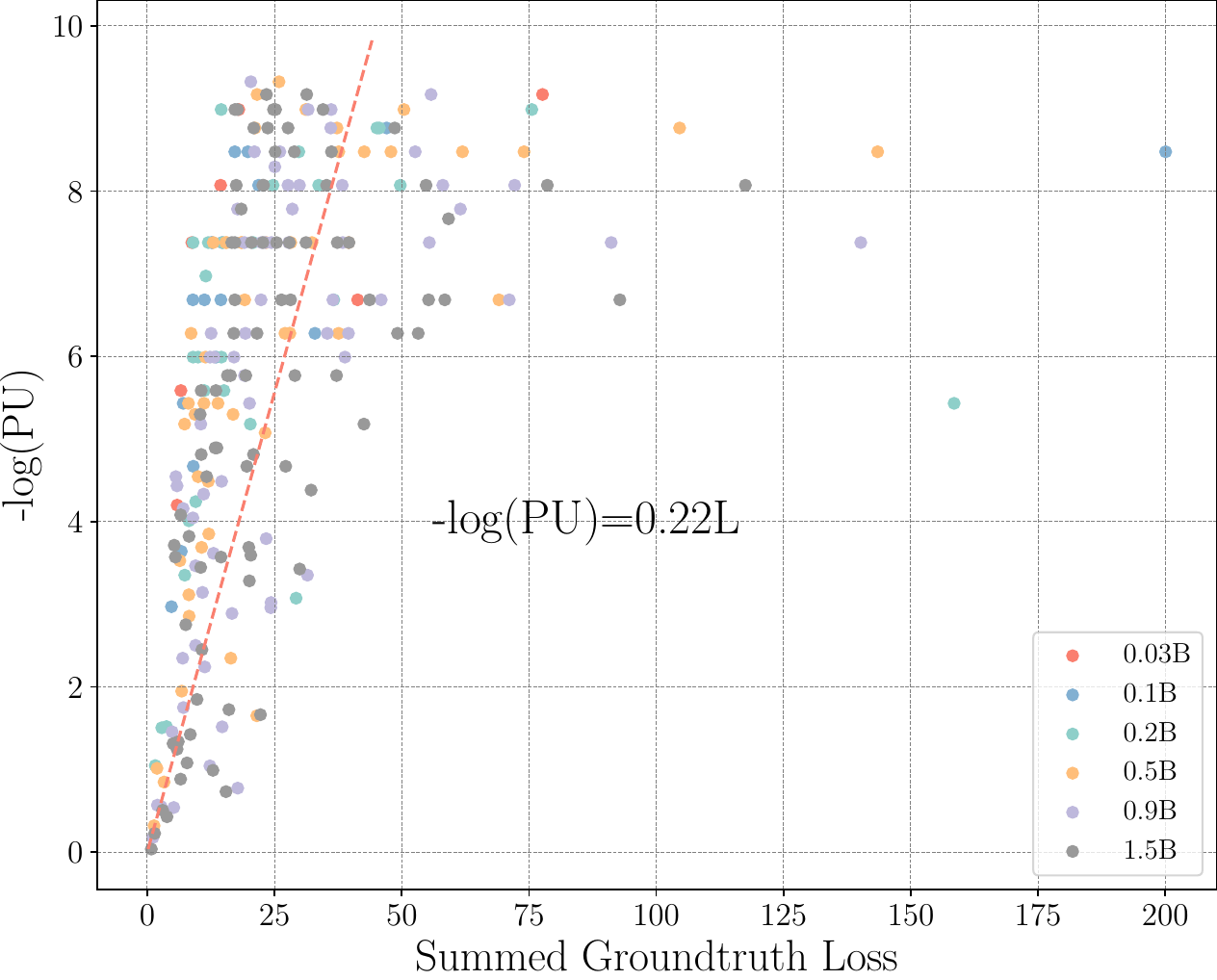}
        \caption{\textsc{PU} w.r.t. the test loss on HumanEval of model series 1.}
        \label{fig:nlpassuntil}
    \end{minipage}
    \hfill
    \vspace{0.1cm}
    \begin{minipage}[b]{0.64\textwidth}
        \centering
        \includegraphics[width=1.02\linewidth]{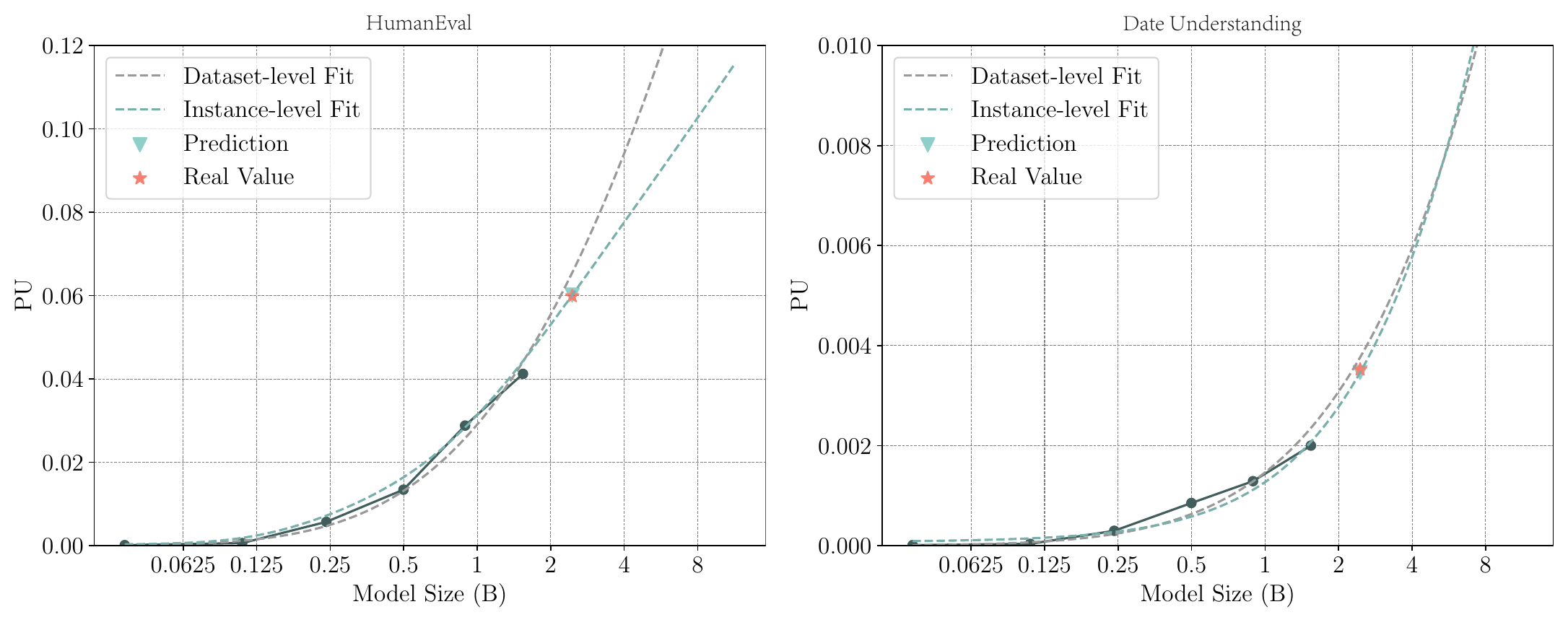}
        \vspace{-0.35cm}
        \caption{We successfully predicted the performance of 2.4B model with 0.05\% deviation (left) and 1.7\% deviation (right).}
        \label{fig:final_curve_humaneval_series1} 
    \end{minipage}
\vspace{-0.5cm}
\end{figure}

\textbf{Estimating \textsc{PassUntil} from Test Loss.}
Estimating at the instance level presents challenges for hard instances that lack adequate non-zero \textsc{PU} values for fitting. These samples may also contribute to \textsc{PU} as the model size increases. We suggest leveraging test loss on ground truth answers to assist the prediction for such instances (\cmt{See Appendix~\ref{app:loss_discuss} for a detailed discussion of its validity}). We leverage the ``easy'' instances, which have both test loss and non-zero \textsc{PU} to estimate the relation between test loss and \textsc{PU} (Figure~\ref{fig:nlpassuntil}). Then we predict the test loss of each instance on 2.4B model based on 0.03B $\sim$ 1.5B models. Finally, we transform the predicted test loss to predicted  \textsc{PU} according to the aforementioned relationship. 
Details are presented in \hypertarget{back:d_2}{Appendix}~\ref{app:test_loss_assist}.  We provide the final prediction result of 2.4B model in Table~\ref{tab:final_fit_result}, and draw the predicted $\textsc{PU}$ curve in Figure~\ref{fig:final_curve_humaneval_series1}.  We can see that the predictions are accurate, with only 0.05\% difference on HumanEval of series 1 and 1.7\% difference on Date Understanding of series 2.

\vspace{-3mm}
\section{Quantitative Analysis of Emergence}
\label{sec:emergent}
\vspace{-3mm}
\cmt{Building on the discovery of the predictability of task performance, we proceed with our investigation into a quantitative analysis of scaling behavior of broader range of tasks. We prove that even with the refined resolution brought by \textsc{PassUntil} and predictability of other emergent abilities, there are still certain abilities hard to be predicted. We establish their mathematical definitions, and examine the possible explanations for such scaling behaviors.
}
\begin{wrapfigure}{r}{0.5\textwidth}
\vspace{-3mm}
    \centering
    \begin{minipage}{.49\textwidth}
        \centering
        \includegraphics[width=1\linewidth]{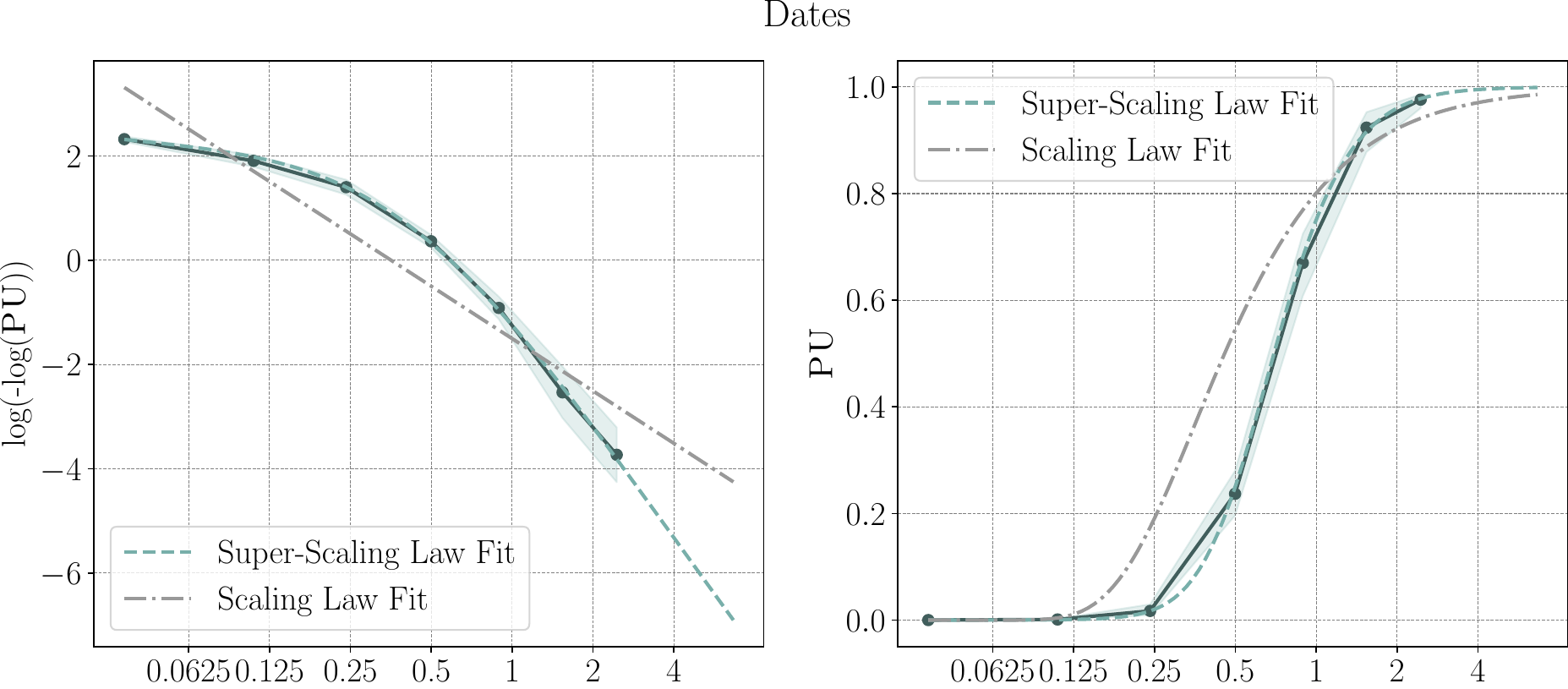}
    \end{minipage}
    \begin{minipage}{.49\textwidth}
        \centering
        \includegraphics[width=1\linewidth]{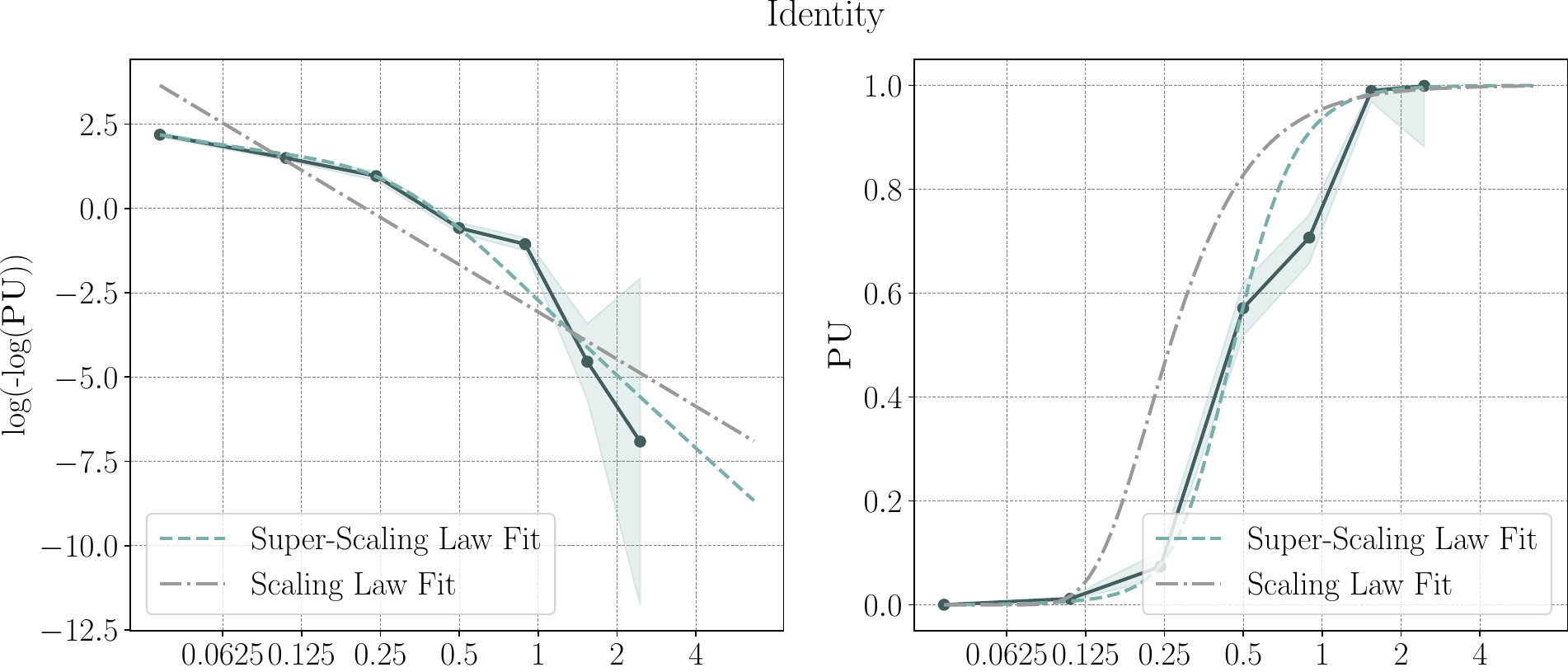}
        \caption{\cmt{Scaling curve for task ``Dates'' and ``Identity''.  Concave functions are observed between $\operatorname{log}(-\operatorname{log}(\textsc{PU}))$ and $\operatorname{log}N$. Scaling law fit curves are in {\color[rgb]{0.55, 0.6, 0.6}grey} and super-scaling law fit curves are in {\color[rgb]{0.3, 0.6, 0.45} green}. }}
        \label{fig:unnatural}
    \end{minipage}
    \hfill %
    \begin{minipage}{.49\textwidth}
    \centering
    \includegraphics[width=1\linewidth]{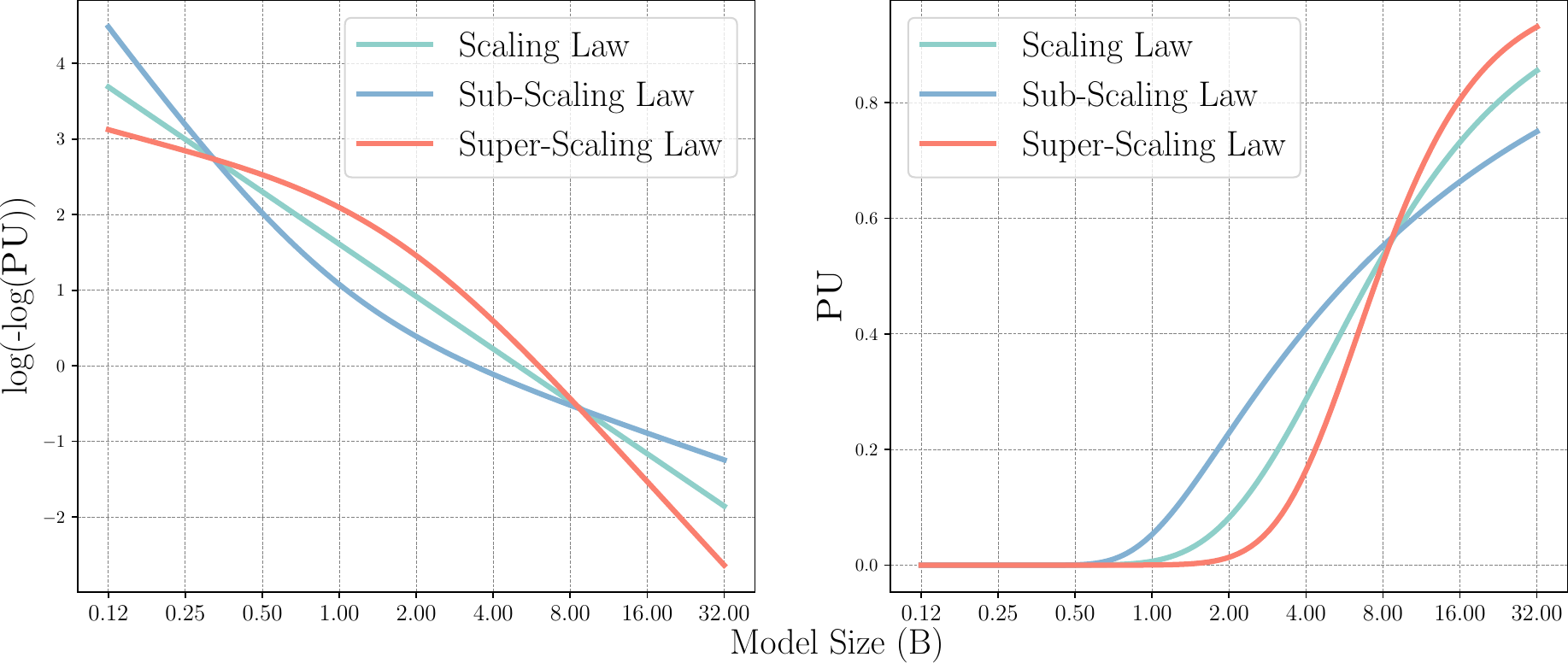}
    \caption{Three basic types of scaling curve, corresponding to convex, linear, and concave function between $\log(-\log(\textsc{PU}))$ and $\log N$.}
    \label{fig:threekindsofgrowth}
\end{minipage}%
\vspace{-2mm}
\end{wrapfigure}

We study the scaling curve on the ``Unnatural In-context Learning (UICL)'' categories in BigBench~\citep{srivastava2022beyond}. ``Unnatural In-context Learning" is a set of 8 tasks designed to specifically study the in-context learning ability. These tasks involve input-output pairs that have been intentionally altered to deviate from the typical training distribution, thereby necessitating the model's focus on unconventional in-context patterns. Task details and examples are in \hypertarget{back:c_4_4}{Appendix}~\ref{app:unantural_in_context_learning}. We randomly select 20 questions in the test set from each task and sample 4-shot examples from the remaining questions to serve as in-context examples. The evaluation metric employed is the exact match, and the upper bound sampling time is set to $10^5$. When fitting the scaling curve, we only utilize the dataset-level \textsc{PassUntil} since these test instances are manually constructed to test one skill of LLM and thus might be devoid of difficulty variation. Since our test set is small, we bootstrap 100 times from the 20 question's test result and use the bootstrapped to calculate the standard error of each \textsc{PassUntil} estimate (shown in the green hue in the Figures).

\textbf{Categorization of Emergence.}  The evaluation on task ``Dates'' and ``Identity'' is shown in Figure~\ref{fig:unnatural}. Other tasks are shown in \hypertarget{back:d_3}{Appendix}~\ref{app:unnaturalincontextscalingcurve}. ``Dates'' exhibit very smooth and consistent improvement starting from 0.03B, while the other tasks are a bit twisty. Nevertheless, 5/8 of these in-context learning tasks display a strictly concave function between $\log(-\log(\textsc{PU}))$ and $\log N$. The others (3/8) miss 1 or 2 valid estimation points due to their extreme difficulty for 0.03B and 0.1B models, since 0 \textsc{PassUntil} is obverseved even with $10^5$ sampling time, which we left for future exploration. The 5/8 tasks deviates from the scaling law (Eq.(\ref{eq:task_scaling_raw})) which requires this function to be linear. This means, unlike those tasks governed by the task scaling law,  where ``growth speed'' $\alpha$ is uniform across different model sizes, \textbf{there exist some tasks that see an increase in ``growth speed'' $\alpha$ as models enlarge. This phenomenon exemplifies an accelerated emergence phenomenon. } To provide concrete discussion of accelerated emergence, we provide our categorization of task scaling curves first.

\textbf{Mathematical Definition of Emergence.} Since the loss scaling law of Eq.(\ref{eq:loss_scaling_law}) is the only widely accepted principle during model scaling, we rely on its derived task scaling law of Eq.(\ref{eq:task_scaling_raw}) as a separator between emergence and other scaling behavior.

\newtheorem{definition}{Definition}
\begin{definition}
Given a spectrum of models, we let the number of non-embedding parameters be variable $N$, suppose the  $\textsc{PU}(N)$ estimated by \textsc{PassUntil} on a task is a continuous function of $N$. Define $F(N) = \operatorname{log}(-\operatorname{log}(\textsc{PU}(N)))$, \cmt{then the scaling curve of a task can be categorized into three basic main categories~\footnote{if $F(N)$ has both convex and concave parts, then we can call it mixed growth. }:}
\begin{enumerate}
    \item if $F(N)$ is a linear function of $\operatorname{log}N$, then the task obeys scaling law growth.
    \item if $F(N)$ is a convex function of $\operatorname{log}N$, then the task obeys sub-scaling law growth.
    \item if $F(N)$ is a concave function of $\operatorname{log}N$, then the task obeys super-scaling law growth, {or ``accelerated emergence''. }
\end{enumerate}
\end{definition}

 Figure~\ref{fig:threekindsofgrowth} shows visualizations of three types of growth. Qualitatively, the scaling curves of all three types appear analogous to exponential growth when performance starts to become noticeable. However, they are qualitatively different. Task scaling curves with task scaling law growth or sub-scaling law growth are easier to predict and control, whereas accelerated emergence is not easy to predict, which might go out of control when the model gets larger. 

\cmt{\textbf{Cause of Shape of Scaling Curve.} The above mathematical definition provides us the opportunity to examine the hypothesis regarding the genesis of these scaling behavior. Here, we first study the following hypothesis: Emergent abilities may be induced by multi-step reasoning~\citep{srivastava2022beyond, wei2022emergent, schaeffer2023emergent}. }
   
We prove that, surprisingly, \textbf{``multi-step reasoning'' leads to sub-scaling law growth}.

\begin{theorem}
Suppose each reasoning step's success rate, measured by \textsc{PassUntil} obeys the scaling law growth, then the multi-step success rate follows the sub-scaling law growth.
\end{theorem}
\vspace{-0.4cm}
\begin{proof}
Suppose the success rate of reasoning step $i$ obeys a scaling law growth with coefficient $c_i$ and $\alpha_i$, then $
 F(N) = \operatorname{log}\left(\sum_i c_i \operatorname{exp}\left(-\alpha_i \operatorname{log} N \right)\right) $.
Using Cauchy–Schwarz inequality, we can prove that $
\frac{\partial^2{F}}{\partial{(\operatorname{log}N )^2}} \geq 0 $. Therefore, the scaling curve is convex. See  \hypertarget{back:b_1}{Appendix} \ref{app:theorectical_analysis} for more.
\end{proof}

\vspace{-0.2cm}

This proof can also be understood more intuitively: the growth speed will initially be boosted by the improvement of those easy steps, and eventually be bounded by the most difficult steps, thus showing a decreasing growth speed. Then, we propose an alternative hypothesis: suggesting that multiple neural ``circuits''~\citep{elhage2021mathematical} may be represented within the LLMs, and that as long as one such circuit can successfully solve the test instance, the test instance is deemed passed.  This hypothesis is inspired by the explanation of ``grokking'' phenomenon given by~\cite{varma2023explaining}. They propose that there exists a memorization circuit and a generalization circuit inside the transformers, and the ``grokking'' phenomenon is led by the generalization circuit getting more efficient than the memorization circuit during training. We will demonstrate that with this hypothesis, the scaling curve exhibits characteristics of emergence.

\begin{theorem}
Suppose multiple circuits $i$ exist in the LLMs that are responsible for solving the task, and each displays scaling law growth and has \textsc{PU}$_i$. \cmt{And suppose the success rate} of the task is the majority voting of these circuits, i.e., $
 F(N)  = \operatorname{log}\left(-\operatorname{log}\max_i \textsc{PU}_i\right)$.
Then, $F(N)$ is a concave function of $\operatorname{log}N$. 
\end{theorem}
\vspace{-0.3cm}
\begin{proof}
    $F(N) =  \min_i (\operatorname{log} c_i-\alpha_i \operatorname{log} N)$. Since the minimum operator keeps concavity, $F(N)$ is a concave function of $\operatorname{log}N$. See Appendix \ref{app:theorectical_analysis} for a more elaborated proof.
\end{proof}

We loosely test the hypothesis by fitting the scaling curve for the UICL task. In practice, similar to~\cite{varma2023explaining}, we adopt a soft version of the majority voting.  We apply a weighted combination between two circuits. And we assume the number of the circuits is 2. Therefore, we fit ${w_1}({\alpha_1}\log N-\log{c_1}) + {w_2}({\alpha_2}\log N - \log c_2)$ to $F(N)$, where $w_1$ and $w_2$ is given by the Softmax of ${\alpha_i}\log N -\log {c_i}$. The resulting fit curve is demonstrated in the {\color[rgb]{0.3, 0.6, 0.45}green} line in Figure~\ref{fig:unnatural} and Appendix~\ref{app:unnaturalincontextscalingcurve}. We can see that 
this hypothesis produces fit curves that align more accurately with the observed performance scaling curve.

\vspace{-2mm}
\section{Conclusion.}
Our work introduces a novel evaluation strategy capable of detecting minimal performance improvements during model scaling, thus opening avenues for quantitatively measuring the task scaling laws and the emergence abilities. This method has enabled the successful prediction of the task performance of larger models. Additionally, we have performed a quantitative analysis of emergent abilities, providing a clearer insight into their nature and origination. This research not only enhances our understanding of LLMs’ scaling properties but also sets the stage for future explorations in scientific scale-up of LLMs.


\section*{Ethical Statement}
In this paper, we demonstrate that although we can predict a set of emergent abilities, the accelerated emergence remains hard to be predicted. The hypothesis regarding the cause of accelerated emergence implies that we need a better understanding of the working mechanism to produce accurate predictions for such emergent ability. Without an understanding of the working mechanism, any fit curve to the early stage of task performance improvement might be governed by another stronger, yet unknown, ``generalization'' circuit when the model gets sufficiently large. Thus, this hypothesis calls for deeper research into the mechanism of LLMs to prevent the safety concerns brought by accelerated emergent abilities.

\section*{Reproducibility Statement}
We will open-source and all evaluation scripts for reference.

\section*{Acknowledgements}
This work is supported by the National Key R\&D Program of China (No.2022ZD0160501).



\bibliography{iclr2023_conference}

\begin{thebibliography}{35}
\providecommand{\natexlab}[1]{#1}
\providecommand{\url}[1]{\texttt{#1}}
\expandafter\ifx\csname urlstyle\endcsname\relax
  \providecommand{\doi}[1]{doi: #1}\else
  \providecommand{\doi}{doi: \begingroup \urlstyle{rm}\Url}\fi

\bibitem[Bird et~al.(2009)Bird, Klein, and Loper]{bird2009natural}
Steven Bird, Ewan Klein, and Edward Loper.
\newblock \emph{Natural language processing with Python: analyzing text with
  the natural language toolkit}.
\newblock " O'Reilly Media, Inc.", 2009.

\bibitem[Brown et~al.(2020)Brown, Mann, Ryder, Subbiah, Kaplan, Dhariwal,
  Neelakantan, Shyam, Sastry, Askell, et~al.]{brown2020language}
Tom Brown, Benjamin Mann, Nick Ryder, Melanie Subbiah, Jared~D Kaplan, Prafulla
  Dhariwal, Arvind Neelakantan, Pranav Shyam, Girish Sastry, Amanda Askell,
  et~al.
\newblock Language models are few-shot learners.
\newblock \emph{Advances in neural information processing systems},
  33:\penalty0 1877--1901, 2020.

\bibitem[Chen et~al.(2021)Chen, Tworek, Jun, Yuan, Pinto, Kaplan, Edwards,
  Burda, Joseph, Brockman, et~al.]{chen2021evaluating}
Mark Chen, Jerry Tworek, Heewoo Jun, Qiming Yuan, Henrique Ponde de~Oliveira
  Pinto, Jared Kaplan, Harri Edwards, Yuri Burda, Nicholas Joseph, Greg
  Brockman, et~al.
\newblock Evaluating large language models trained on code.
\newblock \emph{arXiv preprint arXiv:2107.03374}, 2021.

\bibitem[Chowdhery et~al.(2022)Chowdhery, Narang, Devlin, Bosma, Mishra,
  Roberts, Barham, Chung, Sutton, Gehrmann, et~al.]{chowdhery2022palm}
Aakanksha Chowdhery, Sharan Narang, Jacob Devlin, Maarten Bosma, Gaurav Mishra,
  Adam Roberts, Paul Barham, Hyung~Won Chung, Charles Sutton, Sebastian
  Gehrmann, et~al.
\newblock Palm: Scaling language modeling with pathways.
\newblock \emph{arXiv preprint arXiv:2204.02311}, 2022.

\bibitem[Devlin et~al.(2018)Devlin, Chang, Lee, and Toutanova]{devlin2018bert}
Jacob Devlin, Ming-Wei Chang, Kenton Lee, and Kristina Toutanova.
\newblock Bert: Pre-training of deep bidirectional transformers for language
  understanding.
\newblock \emph{arXiv preprint arXiv:1810.04805}, 2018.

\bibitem[Dubois et~al.(2023)Dubois, Li, Taori, Zhang, Gulrajani, Ba, Guestrin,
  Liang, and Hashimoto]{dubois2023alpacafarm}
Yann Dubois, Xuechen Li, Rohan Taori, Tianyi Zhang, Ishaan Gulrajani, Jimmy Ba,
  Carlos Guestrin, Percy Liang, and Tatsunori~B Hashimoto.
\newblock Alpacafarm: A simulation framework for methods that learn from human
  feedback.
\newblock \emph{arXiv preprint arXiv:2305.14387}, 2023.

\bibitem[Ganguli et~al.(2022)Ganguli, Hernandez, Lovitt, Askell, Bai, Chen,
  Conerly, Dassarma, Drain, Elhage, et~al.]{ganguli2022predictability}
Deep Ganguli, Danny Hernandez, Liane Lovitt, Amanda Askell, Yuntao Bai, Anna
  Chen, Tom Conerly, Nova Dassarma, Dawn Drain, Nelson Elhage, et~al.
\newblock Predictability and surprise in large generative models.
\newblock In \emph{Proceedings of the 2022 ACM Conference on Fairness,
  Accountability, and Transparency}, pp.\  1747--1764, 2022.

\bibitem[Gao et~al.(2020)Gao, Biderman, Black, Golding, Hoppe, Foster, Phang,
  He, Thite, Nabeshima, et~al.]{gao2020pile}
Leo Gao, Stella Biderman, Sid Black, Laurence Golding, Travis Hoppe, Charles
  Foster, Jason Phang, Horace He, Anish Thite, Noa Nabeshima, et~al.
\newblock The pile: An 800gb dataset of diverse text for language modeling.
\newblock \emph{arXiv preprint arXiv:2101.00027}, 2020.

\bibitem[Gao et~al.(2023)Gao, Schulman, and Hilton]{gao2023scaling}
Leo Gao, John Schulman, and Jacob Hilton.
\newblock Scaling laws for reward model overoptimization.
\newblock In \emph{International Conference on Machine Learning}, pp.\
  10835--10866. PMLR, 2023.

\bibitem[Hendrycks \& Gimpel(2016)Hendrycks and Gimpel]{hendrycks2016gaussian}
Dan Hendrycks and Kevin Gimpel.
\newblock Gaussian error linear units (gelus).
\newblock \emph{arXiv preprint arXiv:1606.08415}, 2016.

\bibitem[Hendrycks et~al.(2020)Hendrycks, Burns, Basart, Zou, Mazeika, Song,
  and Steinhardt]{hendrycks2020measuring}
Dan Hendrycks, Collin Burns, Steven Basart, Andy Zou, Mantas Mazeika, Dawn
  Song, and Jacob Steinhardt.
\newblock Measuring massive multitask language understanding.
\newblock \emph{arXiv preprint arXiv:2009.03300}, 2020.

\bibitem[Henighan et~al.(2020)Henighan, Kaplan, Katz, Chen, Hesse, Jackson,
  Jun, Brown, Dhariwal, Gray, et~al.]{henighan2020scaling}
Tom Henighan, Jared Kaplan, Mor Katz, Mark Chen, Christopher Hesse, Jacob
  Jackson, Heewoo Jun, Tom~B Brown, Prafulla Dhariwal, Scott Gray, et~al.
\newblock Scaling laws for autoregressive generative modeling.
\newblock \emph{arXiv preprint arXiv:2010.14701}, 2020.

\bibitem[Hestness et~al.(2017)Hestness, Narang, Ardalani, Diamos, Jun,
  Kianinejad, Patwary, Yang, and Zhou]{hestness2017deep}
Joel Hestness, Sharan Narang, Newsha Ardalani, Gregory Diamos, Heewoo Jun,
  Hassan Kianinejad, Md~Mostofa~Ali Patwary, Yang Yang, and Yanqi Zhou.
\newblock Deep learning scaling is predictable, empirically.
\newblock \emph{arXiv preprint arXiv:1712.00409}, 2017.

\bibitem[Hoffmann et~al.(2022)Hoffmann, Borgeaud, Mensch, Buchatskaya, Cai,
  Rutherford, Casas, Hendricks, Welbl, Clark, et~al.]{hoffmann2022training}
Jordan Hoffmann, Sebastian Borgeaud, Arthur Mensch, Elena Buchatskaya, Trevor
  Cai, Eliza Rutherford, Diego de~Las Casas, Lisa~Anne Hendricks, Johannes
  Welbl, Aidan Clark, et~al.
\newblock Training compute-optimal large language models.
\newblock \emph{arXiv preprint arXiv:2203.15556}, 2022.

\bibitem[Kaplan et~al.(2020)Kaplan, McCandlish, Henighan, Brown, Chess, Child,
  Gray, Radford, Wu, and Amodei]{kaplan2020scaling}
Jared Kaplan, Sam McCandlish, Tom Henighan, Tom~B Brown, Benjamin Chess, Rewon
  Child, Scott Gray, Alec Radford, Jeffrey Wu, and Dario Amodei.
\newblock Scaling laws for neural language models.
\newblock \emph{arXiv preprint arXiv:2001.08361}, 2020.

\bibitem[Li et~al.(2023{\natexlab{a}})Li, Allal, Zi, Muennighoff, Kocetkov,
  Mou, Marone, Akiki, Li, Chim, et~al.]{li2023starcoder}
Raymond Li, Loubna~Ben Allal, Yangtian Zi, Niklas Muennighoff, Denis Kocetkov,
  Chenghao Mou, Marc Marone, Christopher Akiki, Jia Li, Jenny Chim, et~al.
\newblock Starcoder: may the source be with you!
\newblock \emph{arXiv preprint arXiv:2305.06161}, 2023{\natexlab{a}}.

\bibitem[Li et~al.(2023{\natexlab{b}})Li, Zhang, Dubois, Taori, Gulrajani,
  Guestrin, Liang, and Hashimoto]{alpaca_eval}
Xuechen Li, Tianyi Zhang, Yann Dubois, Rohan Taori, Ishaan Gulrajani, Carlos
  Guestrin, Percy Liang, and Tatsunori~B. Hashimoto.
\newblock Alpacaeval: An automatic evaluator of instruction-following models.
\newblock \url{https://github.com/tatsu-lab/alpaca_eval}, 2023{\natexlab{b}}.

\bibitem[Muennighoff et~al.(2023)Muennighoff, Rush, Barak, Scao, Piktus, Tazi,
  Pyysalo, Wolf, and Raffel]{muennighoff2023scaling}
Niklas Muennighoff, Alexander~M Rush, Boaz Barak, Teven~Le Scao, Aleksandra
  Piktus, Nouamane Tazi, Sampo Pyysalo, Thomas Wolf, and Colin Raffel.
\newblock Scaling data-constrained language models.
\newblock \emph{arXiv preprint arXiv:2305.16264}, 2023.

\bibitem[Nelson et~al.(2021)Nelson, Neel, Catherine, Tom, Nicholas, Ben,
  Amanda, Yuntao, Anna, Tom, Nova, Dawn, Deep, Zac, Danny, Andy, Jackson,
  Liane, Kamal, Dario, Tom, Jack, Jared, Sam, and
  Chris]{elhage2021mathematical}
Elhage Nelson, Nanda Neel, Olsson Catherine, Henighan Tom, Joseph Nicholas,
  Mann Ben, Askell Amanda, Bai Yuntao, Chen Anna, Conerly Tom, DasSarma Nova,
  Drain Dawn, Ganguli Deep, Hatfield-Dodds Zac, Hernandez Danny, Jones Andy,
  Kernion Jackson, Lovitt Liane, Ndousse Kamal, Amodei Dario, Brown Tom, Clark
  Jack, Kaplan Jared, McCandlish Sam, and Olah Chris.
\newblock A mathematical framework for {T}ransformer circuits.
\newblock 2021.
\newblock URL \url{https://transformer-circuits.pub/2021/framework/index.html}.

\bibitem[OpenAI(2023)]{openai2023gpt4}
OpenAI.
\newblock Gpt-4 technical report, 2023.

\bibitem[Power et~al.(2022)Power, Burda, Edwards, Babuschkin, and
  Misra]{power2022grokking}
Alethea Power, Yuri Burda, Harri Edwards, Igor Babuschkin, and Vedant Misra.
\newblock Grokking: Generalization beyond overfitting on small algorithmic
  datasets.
\newblock \emph{arXiv preprint arXiv:2201.02177}, 2022.

\bibitem[Raffel et~al.(2020)Raffel, Shazeer, Roberts, Lee, Narang, Matena,
  Zhou, Li, and Liu]{raffel2020exploring}
Colin Raffel, Noam Shazeer, Adam Roberts, Katherine Lee, Sharan Narang, Michael
  Matena, Yanqi Zhou, Wei Li, and Peter~J Liu.
\newblock Exploring the limits of transfer learning with a unified text-to-text
  transformer.
\newblock \emph{The Journal of Machine Learning Research}, 21\penalty0
  (1):\penalty0 5485--5551, 2020.

\bibitem[Rozi{\`e}re et~al.(2023)Rozi{\`e}re, Gehring, Gloeckle, Sootla, Gat,
  Tan, Adi, Liu, Remez, Rapin, et~al.]{roziere2023code}
Baptiste Rozi{\`e}re, Jonas Gehring, Fabian Gloeckle, Sten Sootla, Itai Gat,
  Xiaoqing~Ellen Tan, Yossi Adi, Jingyu Liu, Tal Remez, J{\'e}r{\'e}my Rapin,
  et~al.
\newblock Code llama: Open foundation models for code.
\newblock \emph{arXiv preprint arXiv:2308.12950}, 2023.

\bibitem[Schaeffer et~al.(2023)Schaeffer, Miranda, and
  Koyejo]{schaeffer2023emergent}
Rylan Schaeffer, Brando Miranda, and Sanmi Koyejo.
\newblock Are emergent abilities of large language models a mirage?
\newblock \emph{arXiv preprint arXiv:2304.15004}, 2023.

\bibitem[Shazeer(2020)]{DBLP:journals/corr/abs-2002-05202}
Noam Shazeer.
\newblock {GLU} variants improve transformer.
\newblock \emph{CoRR}, abs/2002.05202, 2020.
\newblock URL \url{https://arxiv.org/abs/2002.05202}.

\bibitem[Sorscher et~al.(2022)Sorscher, Geirhos, Shekhar, Ganguli, and
  Morcos]{sorscher2022beyond}
Ben Sorscher, Robert Geirhos, Shashank Shekhar, Surya Ganguli, and Ari Morcos.
\newblock Beyond neural scaling laws: beating power law scaling via data
  pruning.
\newblock \emph{Advances in Neural Information Processing Systems},
  35:\penalty0 19523--19536, 2022.

\bibitem[Srivastava et~al.(2022)Srivastava, Rastogi, Rao, Shoeb, Abid, Fisch,
  Brown, Santoro, Gupta, Garriga-Alonso, et~al.]{srivastava2022beyond}
Aarohi Srivastava, Abhinav Rastogi, Abhishek Rao, Abu Awal~Md Shoeb, Abubakar
  Abid, Adam Fisch, Adam~R Brown, Adam Santoro, Aditya Gupta, Adri{\`a}
  Garriga-Alonso, et~al.
\newblock Beyond the imitation game: Quantifying and extrapolating the
  capabilities of language models.
\newblock \emph{arXiv preprint arXiv:2206.04615}, 2022.

\bibitem[Touvron et~al.(2023{\natexlab{a}})Touvron, Lavril, Izacard, Martinet,
  Lachaux, Lacroix, Rozi{\`e}re, Goyal, Hambro, Azhar,
  et~al.]{touvron2023llama}
Hugo Touvron, Thibaut Lavril, Gautier Izacard, Xavier Martinet, Marie-Anne
  Lachaux, Timoth{\'e}e Lacroix, Baptiste Rozi{\`e}re, Naman Goyal, Eric
  Hambro, Faisal Azhar, et~al.
\newblock Llama: Open and efficient foundation language models.
\newblock \emph{arXiv preprint arXiv:2302.13971}, 2023{\natexlab{a}}.

\bibitem[Touvron et~al.(2023{\natexlab{b}})Touvron, Martin, Stone, Albert,
  Almahairi, Babaei, Bashlykov, Batra, Bhargava, Bhosale,
  et~al.]{touvron2023llama2}
Hugo Touvron, Louis Martin, Kevin Stone, Peter Albert, Amjad Almahairi, Yasmine
  Babaei, Nikolay Bashlykov, Soumya Batra, Prajjwal Bhargava, Shruti Bhosale,
  et~al.
\newblock Llama 2: Open foundation and fine-tuned chat models.
\newblock \emph{arXiv preprint arXiv:2307.09288}, 2023{\natexlab{b}}.

\bibitem[Varma et~al.(2023)Varma, Shah, Kenton, Kram{\'a}r, and
  Kumar]{varma2023explaining}
Vikrant Varma, Rohin Shah, Zachary Kenton, J{\'a}nos Kram{\'a}r, and Ramana
  Kumar.
\newblock Explaining grokking through circuit efficiency.
\newblock \emph{arXiv preprint arXiv:2309.02390}, 2023.

\bibitem[Wei et~al.(2022{\natexlab{a}})Wei, Tay, Bommasani, Raffel, Zoph,
  Borgeaud, Yogatama, Bosma, Zhou, Metzler, et~al.]{wei2022emergent}
Jason Wei, Yi~Tay, Rishi Bommasani, Colin Raffel, Barret Zoph, Sebastian
  Borgeaud, Dani Yogatama, Maarten Bosma, Denny Zhou, Donald Metzler, et~al.
\newblock Emergent abilities of large language models.
\newblock \emph{arXiv preprint arXiv:2206.07682}, 2022{\natexlab{a}}.

\bibitem[Wei et~al.(2022{\natexlab{b}})Wei, Wang, Schuurmans, Bosma, Xia, Chi,
  Le, Zhou, et~al.]{wei2022chain}
Jason Wei, Xuezhi Wang, Dale Schuurmans, Maarten Bosma, Fei Xia, Ed~Chi, Quoc~V
  Le, Denny Zhou, et~al.
\newblock Chain-of-thought prompting elicits reasoning in large language
  models.
\newblock \emph{Advances in Neural Information Processing Systems},
  35:\penalty0 24824--24837, 2022{\natexlab{b}}.

\bibitem[Yang et~al.(2022)Yang, Hu, Babuschkin, Sidor, Liu, Farhi, Ryder,
  Pachocki, Chen, and Gao]{yang2022tensor}
Greg Yang, Edward~J Hu, Igor Babuschkin, Szymon Sidor, Xiaodong Liu, David
  Farhi, Nick Ryder, Jakub Pachocki, Weizhu Chen, and Jianfeng Gao.
\newblock Tensor programs v: Tuning large neural networks via zero-shot
  hyperparameter transfer.
\newblock \emph{arXiv preprint arXiv:2203.03466}, 2022.

\bibitem[Yao \& Wang(2023)Yao and Wang]{yao2023research}
Yiqun Yao and Yequan Wang.
\newblock Research without re-search: Maximal update parametrization yields
  accurate loss prediction across scales.
\newblock \emph{arXiv preprint arXiv:2304.06875}, 2023.

\bibitem[Zhai et~al.(2022)Zhai, Kolesnikov, Houlsby, and
  Beyer]{zhai2022scaling}
Xiaohua Zhai, Alexander Kolesnikov, Neil Houlsby, and Lucas Beyer.
\newblock Scaling vision transformers.
\newblock In \emph{Proceedings of the IEEE/CVF Conference on Computer Vision
  and Pattern Recognition}, pp.\  12104--12113, 2022.

\end{thebibliography}
\bibliographystyle{iclr2023_conference}

\clearpage

\appendix

\textbf{Note: clicking each {\faHandPointerO} in the appendix will allow you to jump back to the corresponding position in the main paper to continue reading. }

\section{Discussion}
\subsection{Limitations}
Our work has several limitations. 
\cmt{
\begin{enumerate}
    \item \textbf{Scale Limitation.} Firstly, we  currently do not extend the prediction of task performance to much larger models (e.g., 10B and more). We will try to scale up the experiment in the future.
    \item \textbf{Scope Limitation.} Secondly, we are not claiming that we can accurately predict the task performance on all tasks. For example, we only \textit{fit} the scaling curve for the tasks that display emergence. We still have a long way to go before we can \textit{predict} these tasks. Even for the tasks that might not display ``emergence'', we currently do not complete a thorough prediction for them. We will add predictions on more of these tasks in the future. That said, predictable scaling, as OpenAI points out~\citep{openai2023gpt4}, is still a very challenging and aspirational goal for AI researchers. Our work serves as the initial attempt to it. 
    \item \textbf{Explanation Limitation.} Thirdly, although we propose a hypothesis regarding the cause of accelerated emergence, our validation for the hypothesis is superficial. We satisfactorily fit the scaling curve under this hypothesis. However, whether this hypothesis is true from the underlying mechanism remains unknown.
\end{enumerate} 
}

\subsection{Discuss of the Use of Loss as an Assistance Metric}
\label{app:loss_discuss}

\cmt{In our experiments of Individual \textsc{PassUntil}, we use loss on ground truth as an assistance to \textsc{PassUntil}, which may raise a misunderstanding: why don't you directly use loss to predict the performance? We provide a detailed illustration below. }

\begin{enumerate}
    \item \cmt{It's important to distinguish between ``loss is not predictive of task performance'' and ``loss can help predict task performance.'' The former suggests that loss is a not sufficient statistic for estimating task performance without other measurement, while the latter indicates that loss is one of useful factors in improving prediction accuracy. In our paper, we clearly verify both statements. Without utilizing the PassUntil method, one cannot deduce actual performance (accuracy) solely from loss values. For example, a loss of 1.0 does not directly translate to an accuracy of 0.2 for a task. And actual performance must be empirically measured. Furthermore, as shown in Figure~\ref{fig:nlpassuntil}, the loss of an individual sample does not have a one-to-one correlation with PassUntil results, much less with discrete accuracy.}
    \item \cmt{However, loss does provide useful information. Once we measure PassUntil across a large sample set, we can establish a statistical relationship between loss and PassUntil (not possible if we only rely on loss data). This relationship can enhance our prediction accuracy.}
    \item \cmt{The incorporation of loss for improved predictions is driven by practical considerations, such as limited computational resources, rather than being a necessity. Figure~\ref{fig:dpu} demonstrates that even without loss data, we can accurately predict task performance. Imagine a scenario where we can measure every sample with sufficient resolution to ensure each is passed at least once; in such a case, loss data would not be necessary.}
\end{enumerate}

\section{Supplementary Materials for \textsc{PassUntil}}
In this section, we provide some additional comments about our evaluation strategy.
We present our intuition for instance-level \textsc{PassUntil}.
\subsection{Instance-level \textsc{PassUntil} Intuition.}
\label{app:instance_level_fit_illustration}
\hyperlink{back:a_2}{\faHandPointerO}
Table \ref{tab:easyhardinstances} delineates the \textsc{PassUntil} for both an easy and a challenging instance within HumanEval. It was observed that with an increase in model size, the easier instance (index 24) exhibited a higher PU. However, the more challenging instance (index 20) continued to manifest trivial performance, suggesting a potential variance in their respective scaling curves. Blindly averaging performance over instances will make the improvement on hard instances vanish compared to the easy ones, leading to an inaccurate prediction after the model gets saturated in the easy instances.
\begin{table}[h]
    \centering
    \begin{tabular}{c|cccccc}
    \toprule
        \multirow{2}{*}{\textbf{Instance index}}  & \multicolumn{6}{c}{\textsc{PassUntil}}   \\
        \cline{2-7}
        & \textbf{0.03B} & \textbf{0.1B} & \textbf{0.2B} & \textbf{0.5B} & \textbf{0.9B} & \textbf{1.5B} \\
        \midrule
        20 & 0 & 0 & 0 & 0.000625 & 0.001875 & 0.008125 \\
        24 & 0.00375 & 0.05125 & 0.350625 & 0.3625 & 0.568125 & 0.796875 \\
         \bottomrule
    \end{tabular}
    \caption{In HumanEval,  an easy instance (index 24) gets a much higher $\textsc{PU}$ compared to the hard one (index 20).}
    \label{tab:easyhardinstances}
\vspace{-0.3cm}
\end{table}

\section{Supplementary Materials on Emergent Abilities}

\subsection{Theoretical Analysis of Hypothesis}
\label{app:theorectical_analysis}
\hyperlink{back:b_1}{\faHandPointerO}
We present the proof of two theorems about the cause of emergent abilities in Section~\ref{sec:emergent} briefly. In this section, we provide the elaborated proofs.

\setcounter{theorem}{1}

\begin{theorem}
Suppose the success rate of each reasoning step $i$, measured by \textsc{PassUntil}, obeys the scaling law growth. Then the multi-step's success rate follows the sub-scaling law growth.
\end{theorem}

\begin{proof}
Suppose the \textsc{PU} of reasoning step $i$ obeys a scaling law growth with coefficient $c_i$ and $\alpha_i$, The overall success rate is 
\begin{equation}
\begin{split}
 F(N) & = \operatorname{log}\left(-\operatorname{log}\prod_i P_i\right)\\ & = \operatorname{log}\left(-\operatorname{log} \prod_i\operatorname{exp}\left(- c_i \operatorname{exp}(-\alpha_i \operatorname{log} N)\right)\right) \\ & = \operatorname{log}\left(\sum_i c_i \operatorname{exp}\left(-\alpha_i \operatorname{log} N \right)\right)\\
\end{split}
\end{equation}
Then we take the second derivative of the $F(N)$ over $\log N$, we can get
\begin{equation}
\label{eq:app_th_1}
\begin{split}
    \frac{\partial^2{F}}{\partial{(\operatorname{log}N )^2}} & = \frac{\sum_i \alpha_i^2 c_i \exp(-\alpha_i \log N) \sum_i c_i \exp(-\alpha_i \log N)}{(\sum_i c_i\exp(-\alpha_i \log N))^2} \\
& - \frac{(\sum_i \alpha_i c_i \exp(-\alpha_i \log N) )^2}{(\sum_i c_i\exp(-\alpha_i \log N))^2}
\end{split}
\end{equation}
Let $k_i = c_i \exp(-\alpha_i \log N) > 0$, the Eq.(\ref{eq:app_th_1}) is
\begin{equation}
\frac{\sum_i \alpha_i^2 k_i\sum_i k_i - (\sum_i \alpha_i k_i)^2}{(\sum_i k_i)^2}
\end{equation}
Using Cauchy–Schwarz inequality, we can prove that
\begin{align}
\frac{\partial^2{F}}{\partial{(\operatorname{log}N )^2}} \geq 0, \quad \forall  \alpha_i > 0, c_i > 0
\end{align}
Only when $\alpha_i \sqrt{k_i} / \sqrt{k_i} = \operatorname{Constant}$, the equation holds, i.e., when all the steps in the reasoning chain scale with the same speed.
Thus, $F(N)$ is a convex function of $\operatorname{log}N$, and the scaling curve exhibits sub-scaling law growth. 
\end{proof}


\begin{theorem}
Suppose multiple circuits exist in the LLMs that are responsible for solving the task, each displays scaling law growth, the \textsc{PassUntil} of the task is the majority voting of these circuits, i.e., $
 F(N)  = \operatorname{log}\left(-\operatorname{log}\max_i P_i\right) $
Then, $F(N)$ is a concave function of $\operatorname{log}N$. 
\end{theorem}

\begin{proof}
\begin{equation}
\label{eq:app_th_2}
\begin{split}
F(N) &= \operatorname{log}\left(-\operatorname{log}\max_i \operatorname{exp}\left(- c_i \operatorname{exp}(-\alpha_i \operatorname{log} N)\right)\right) \\
    &= \operatorname{log} \min_i  c_i \operatorname{exp}(-\alpha_i \operatorname{log} N) \\
&=  \min_i (\operatorname{log} c_i-\alpha_i \operatorname{log} N) \\
\end{split}
\end{equation}
 Since the minimum operator keeps concavity, $F(N)$ is a concave function of $\operatorname{log}N$. 
\end{proof}

\section{Details of Experimental Configurations}
In this section, we detail the model configurations, training configurations, and data mixtures used for the two series of models.

\subsection{Model Configuration}
\label{app:modelconf}
\hyperlink{back:c_1}{\faHandPointerO}
Table~\ref{tab:model_configs} shows the detailed model configurations and training configuration of the series models in the scaling curve, which aims to keep a uniform ``shape'' while expanding the model size. We use a similar architecture to Llama 2~\citep{touvron2023llama2}. Some minimal differences include: we use tied embedding between the input and output embeddings, and we use gated-GeLU~\citep{hendrycks2016gaussian} instead of gated-SiLU~\citep{DBLP:journals/corr/abs-2002-05202}.

\begin{table}[htbp]
    \centering
    \begin{tabular}{c|cccccccccc}
    \toprule
        Name & i & N (B)& $d_m$ & $d_{ff}$ &$d_h$ & $n_h$ & $L$ & BS (M) & TS & Tokens (B)\\
    \midrule
  $\backslash$ & i & $\backslash$ & $d_hn_h$ & $2.5d_{m}$ & 64 & $\lfloor \frac{i(8+i)}{4} \rfloor$ & $\lfloor4i\rfloor$ & $\backslash$ & $\backslash$ &$\backslash$\\
           0.03B & 3 &  0.036 & 512 & 1280 & 64 & 8 & 12 & 0.33 & 2196 & 0.72\\
          0.1B & 4 & 0.109 & 768 & 1920 & 64 & 12 & 16  & 0.88 & 2464 & 2.18\\
         0.2B& 5 & 0.241 & 1024 & 2560 & 64 & 16 & 20  & 1.57 & 3064 & 4.82\\
        0.5B& 6 & 0.499 & 1344 & 3360 & 64 & 21 & 24 & 2.10 & 4758 & 9.99\\
        0.9B & 7 & 0.892 & 1664 & 4160 & 64 & 26 & 28  & 2.95 & 6049 & 17.9  \\
        1.5B & 8 & 1.542 & 2048 & 5120 & 64 & 32 & 32 & 4.26 & 7230 & 30.8\\
       2.4B & 9 & 2.45 & 2432 & 6080 & 64 & 38 & 36 & 5.51 & 8900 & 49.0\\
    \bottomrule
    \end{tabular}
    \caption{Model configurations and training configurations of the models in the scaling curve. N(B) represents the number of non-embedding parameters of the model, measured in billions. BS(M) indicates the number of tokens in a batch (i.e., batch size) used to train the model, measured in millions. TS denotes the training steps. Tokens(B) refers the total number of tokens used to train the model.}
    \label{tab:model_configs}
\end{table}

\subsection{Pre-training Corpora}
\label{app:datamixture}
\hyperlink{back:c_2}{\faHandPointerO}
We pre-train two series of LLMs using different data mixtures to demonstrate the generality of our experiments.  Tables \ref{tab:series1_data_mixture} and \ref{tab:series2_data_mixture} respectively display the specific data mixture proportions for Series 1 and Series 2 LLMs.

\subsection{Hyper-parameters Study}
\label{app:hyperparameters}
\hyperlink{back:c_hyper}{\faHandPointerO}
\textbf{Learning Rate.} We use a cosine learning rate scheduler, analogous to those in preceding studies~\citep{touvron2023llama, touvron2023llama2, hoffmann2022training}. The maximum learning rate is consistently fixed at $0.01$ across varying model scales, with no significant loss explosion at this rate. This stability is potentially attributed to our normalization strategies~\citep{yang2022tensor} and increased batch size across scales. Echoing findings from \citet{hoffmann2022training}, we ascertain that for training LLMs up to a specific end step, the optimal cycle length of the cosine learning rate scheduler is equivalent to the end step. Deviations from this optimal cycle length, either longer or shorter, result in sub-optimal performance.

\textbf{Batch Size.} To estimate the optimal batch size required for model pre-training, we replicate the experiments in alignment with \citet{kaplan2020scaling} to determine the optimal batch size of a model and adjust the real batch size slightly from the optimal batch size to maximize GPU utility. The values of batch sizes and train steps are listed in Table~\ref{tab:model_configs}.

\subsection{Test Set Configurations}
\label{app:testset}
\hyperlink{back:c_4}{\faHandPointerO}
In this section, we introduce the test sets and evaluation details in our experiments.

\subsubsection{HumanEval}
\label{app:humaneval}
The HumanEval~\citep{chen2021evaluating} dataset released by OpenAI encompasses 164 programming problems. Each problem is composed of a function signature, a docstring, a body, and multiple unit tests. Our assessment of this dataset is conducted utilizing a zero-shot approach. The completion of code, as generated by LLMs, is deemed passed only if it successfully passes all unit tests. For our evaluations, we set the upper bound of sampling times in \textsc{PassUntil} to $10^4$.

\subsubsection{Emoji Movie}
\label{app:emoji_movies}
\hyperlink{back:c_4_2}{\faHandPointerO}
Emoji Movie is a subtask of BigBench~\citep{srivastava2022beyond} and requires LLMs to identify well-known movies from their plots described using emojis. Our evaluation methodology incorporates the use of Chain-of-Thought (CoT) and 4-shot In-context Learning. We randomly select 41 test instances (initially 50 instances, with 9 distracting instances removed, see Appendix~\ref{app:removing_distracting})  to constitute our test set and arbitrarily designate 4 instances as few-shot contexts. For CoT, we use GPT-4 to generate prompts for each instance in the few-shot context. The model is expected to read the 4-shot in-context examples,  generate a thought, and then provide the answer. Our evaluation methodology employs extract string match, i.e. where the output of the model contains the target film name. We set the sampling upper bound times set to be $10^5$.

\begin{wrapfigure}{r}{0.4\textwidth}
        \centering
        \begin{tabular}{c|c}
        \toprule
           Corpora   &  Token Portion \\
        \midrule
        StarCoder\_Python &  0.3\\
        StarCoder\_Others &  0.7\\
    \bottomrule
        \end{tabular}
        \captionof{table}{Pre-training corpora used for scaling the Code LLMs (model series 1).}
        \label{tab:series1_data_mixture}
    \centering
    \begin{tabular}{c|c}
    \toprule
       Corpora   &  Token Portion \\
    \midrule
    StarCoder\_Python &  0.15\\
    StarCoder\_Others &  0.12\\
    Stack\_Overflow & 0.03 \\
    Arxiv & 0.05 \\
    Pile & 0.65 \\
\bottomrule
    \end{tabular}
    \captionof{table}{Pre-training corpora used for scaling the Code-Text LLMs (model series 2).}
    \label{tab:series2_data_mixture}
\end{wrapfigure}

\subsubsection{Date Understanding}
\label{app:dateunderstanding}
\hyperlink{back:c_4_3}{\faHandPointerO}
Date Understanding, a subset of BigBench~\citep{srivastava2022beyond}, is constructed to evaluate the capability of LLMs in comprehending dates, by posing questions related to the date reasoning. 
For the evaluation of this task, we employ a 4-shot In-context Learning. We randomly sample 47 instances to form the test set (initially 50 instances, with 3 distracting instances removed, see Appendix~\ref{app:removing_distracting}). We random sample 4 instances from the remaining dataset to serve as in-context examples. We also use extract string match to measure the output from LLMs and set the sampling upper bound times to $10^5$.

\subsubsection{Unnatural In-context Learning Tasks}
\label{app:unantural_in_context_learning}
\hyperlink{back:c_4_4}{\faHandPointerO}
The Unnatural In-context Learning tasks serve as a series of distinctive subtasks within BigBench~\citep{srivastava2022beyond}. These subtasks are designed to assess the models’ ability to perform in-context learning where the context sequences are intentionally altered to be likely outside of the training distribution, necessitating the model’s attention to unconventional in-context patterns. Some instances of these subtasks are exemplified in Table \ref{tab:unnaturalincontextexample}.
For each task, 20 instances are randomly sampled to compose the test set, utilizing a 4-shot In-context Learning configuration. Four instances are randomly selected from the remaining dataset to provide context. We use extract string match to measure the output from LLMs and set the sampling upper bound times to $10^5$.

\begin{table}[!htbp]
    \centering
    \label{tab:unnaturalicl}
    \begin{tabular}{p{4cm}l}
        \toprule
        \textbf{Task Name}  & \textbf{Example} \\
        \midrule
        Dates  & Input: 2015-10-22\ Target: \textit{!10!22!2015!}\\
        Dates with Unnatural Form & Input: !08!24!1984!\ Target: \textit{1984-08-24}\\
        Dates with Unnatural Content & Input: 96980-49665-10674\ Target: \textit{!49665!10674!96980!}\\
        Dates with Unnatural Form and Content & Input: !49665!10674!96980! \ Target: \textit{96980-49665-10674}\\
        Identity & Input: a, b, c, d, e\ Target: \textit{a, b, c, d, e} \\
        Reverse Natural Content & Input: t, u, o, b, a\ Target: \textit{a, b, o, u, t} \\
        Reverse to Natural Content & Input: r, o, o, m, s\ Target: \textit{s, m, o, o, r}  \\
        2-digits & Input: 10 - 10 = \ Target: \textit{20}\\
        \bottomrule
    \end{tabular}
    \caption{Example Tasks in Unnatural In-context Learning Tasks}
    \label{tab:unnaturalincontextexample}
\end{table}

\subsection{Removing Distracting Factor is Important When Measuring Tiny Performance.}
\label{app:removing_distracting}\hyperlink{back:d_5}{\faHandPointerO}
We notice that removing the distracting factor is important when measuring the minor performance gain during scaling. The distracting factor means that a test instance is drastically different from the other test instance in terms of required abilities or evaluation bias. Note that we select the distracting factor based on the observation of test instances, which does not lead to information leakage when predicting the 2.4B model. 

For Emoji Movie, some of the movie names are common words, enabling even a modestly sized model to ``guess'' them correctly based on our assessment criteria: the determination of model correctness is contingent upon the presence of movie names within the model's output. Figure \ref{fig:emoji_common_words} shows that there is no significant association in the pass rates between models of varied scales. In other words, the scaling law does not have much of an impact on model performance for these problems. Consequently, it becomes essential to exclude such distracting factors from consideration. We remove the movie names that are common words identified by the popular toolkit NLTK~\footnote{\url{https://www.nltk.org/}}.

For Date Understanding, we omit the following instance shown in Table~\ref{tab:du_distracting}. These instances only require the model to extract the answer from the context and don't require reasoning about the date.

In GPT-4 report~\citep{openai2023gpt4}, they split the HumanEval dataset into separate bins with different difficulties and conducted scaling prediction for each bin, thus removing the distraction of easy examples to hard examples.

\begin{table}[!htbp]
\centering
\scalebox{0.9}{
    \begin{tabular}{p{13cm}}
        \toprule
      \textbf{Example} \\
        \midrule
         Today's meeting is rescheduled to 11 am tomorrow, 10/16/1924. What is the date tomorrow in MM/DD/YYYY? \\
         \hline
         Yesterday was 12/31/1929. Today could not be 12/32/1929 because December has only 31 days. What is the date yesterday in MM/DD/YYYY? \\
        \hline
        Today is 9/7. Jane is watching NFL 2003. What is the date today in MM/DD/YYYY? \\ 
    \bottomrule
    \end{tabular}
}
    \caption{Distracting instances in Date Understanding Tasks.}
    \label{tab:du_distracting}
\end{table}

\begin{figure}[h]
    \centering
\includegraphics[width=0.7\linewidth]{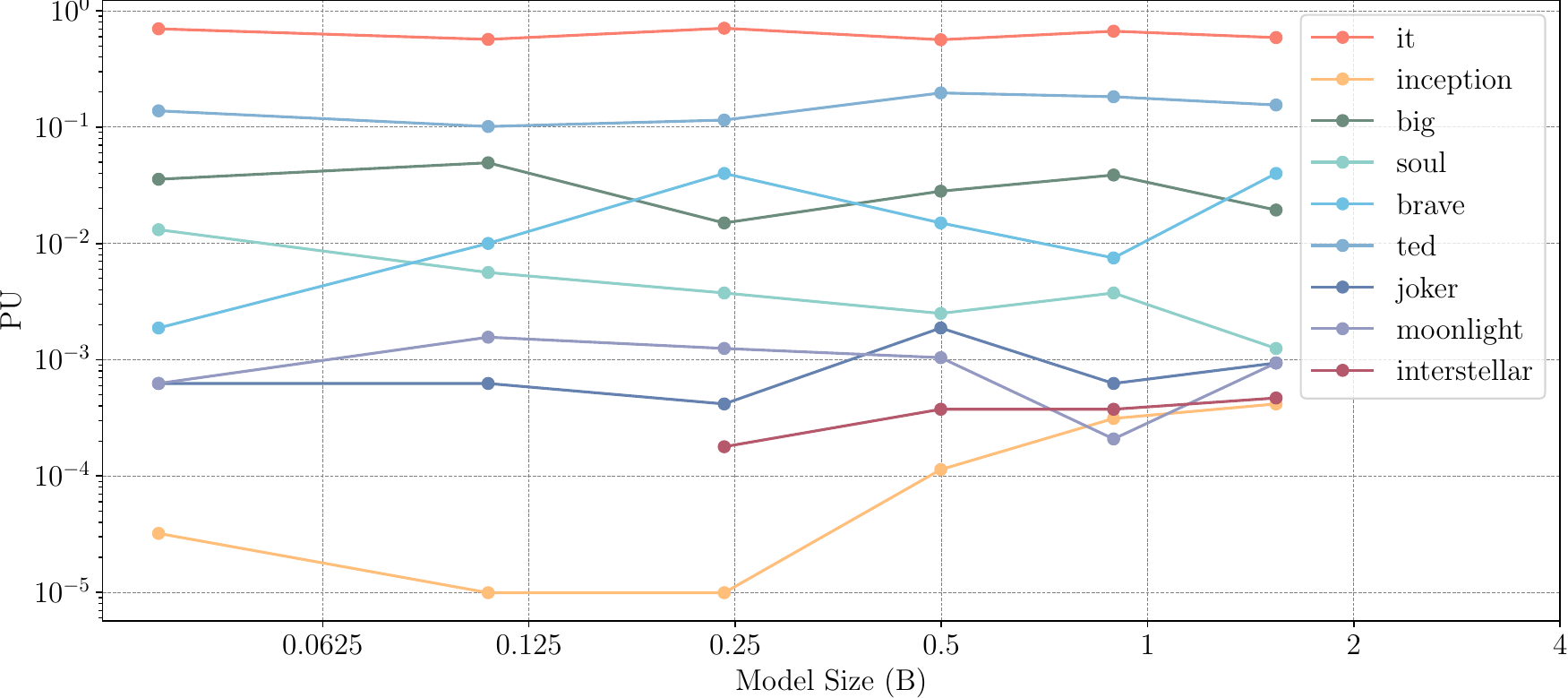}
    \caption{Large and small models have similar \textsc{PU} on these instances (mainly due to randomly sample from the vocabulary space), which creates distracting factors in our experiments.}
    \label{fig:emoji_common_words}
\end{figure}

\section{Additional Experimental Results}
In this section, we display some additional experimental results, including the additional fit curve of dataset level of \textsc{PassUntil}, and the methods of utilizing test loss to assist the instance-level \textsc{PassUntil} estimates.

\subsection{Additional Dataset Level \textsc{PassUntil} result.}
The performances of series 2 models on HumanEval are represented in Figure~\ref{fig:passrate_vs_modelsize-code2}. This prediction is less accurate compared to series 1. However, with instance level \textsc{PassUntil}, the prediction precision improves. 
\begin{figure}[!htbp]
        \centering
        \includegraphics[width=0.7\linewidth]{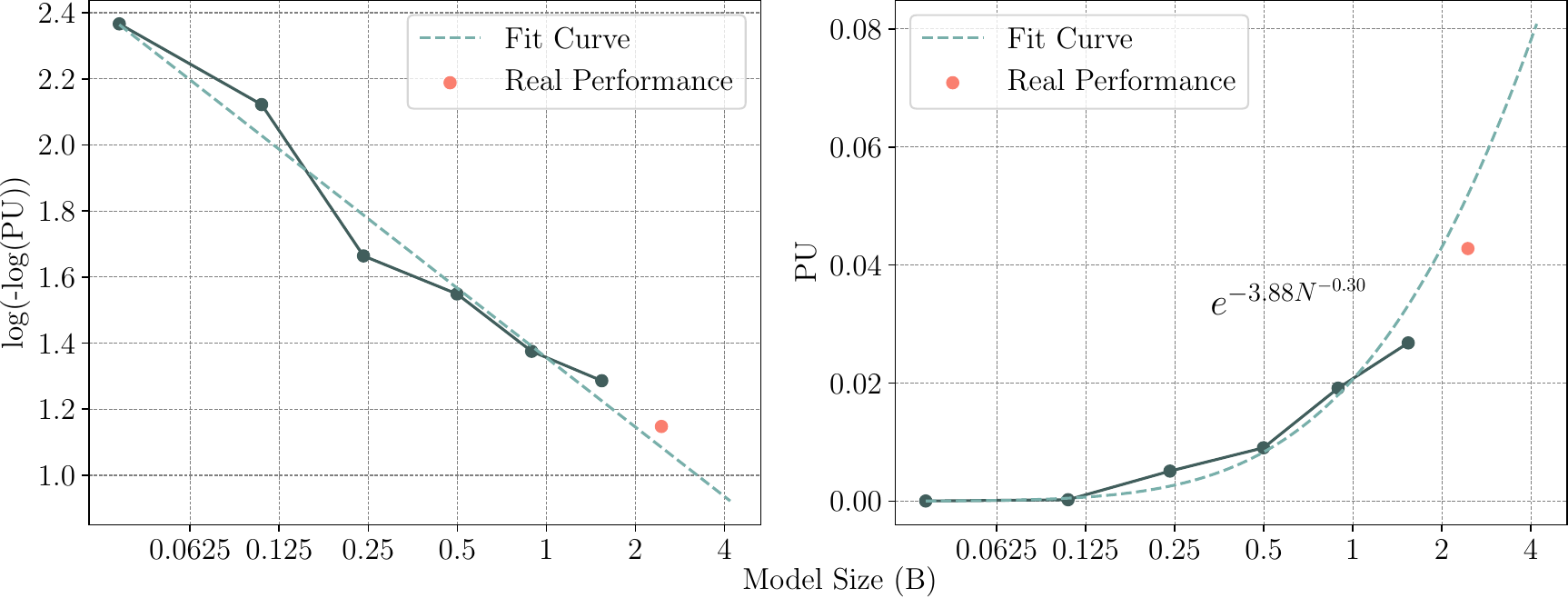}
    \caption{Additional figure on Test Loss Assitted \textsc{PassUntil} Estimate.}
    \label{fig:passrate_vs_modelsize-code2}
\end{figure}

\clearpage

\subsection{Estimating \textsc{PassUntil}  from Test Loss}
\label{app:test_loss_assist}
\hyperlink{back:d_2}{\faHandPointerO}
As shown in Figure \ref{fig:loss_vs_passuntil}, we propose leveraging test loss on ground
truth answers to assist the prediction for ``hard samples". For model series 1 and HumanEval task, the linear relationship is found to be $\textsc{PU} \sim 0.22 L$. For model series 2 and HumanEval task, the linear relationship is found to be $\textsc{PU} \sim 0.23 L$. For model series 2 and Date Understanding task, the linear relationship is found to be $\textsc{PU} \sim 0.96 L$. And for model series 2 and Emoji Movie task, the linear relationship is found to be $\textsc{PU} \sim 0.43 L$.

\begin{figure}[!htbp]
        \centering
        \includegraphics[width=0.65\linewidth]{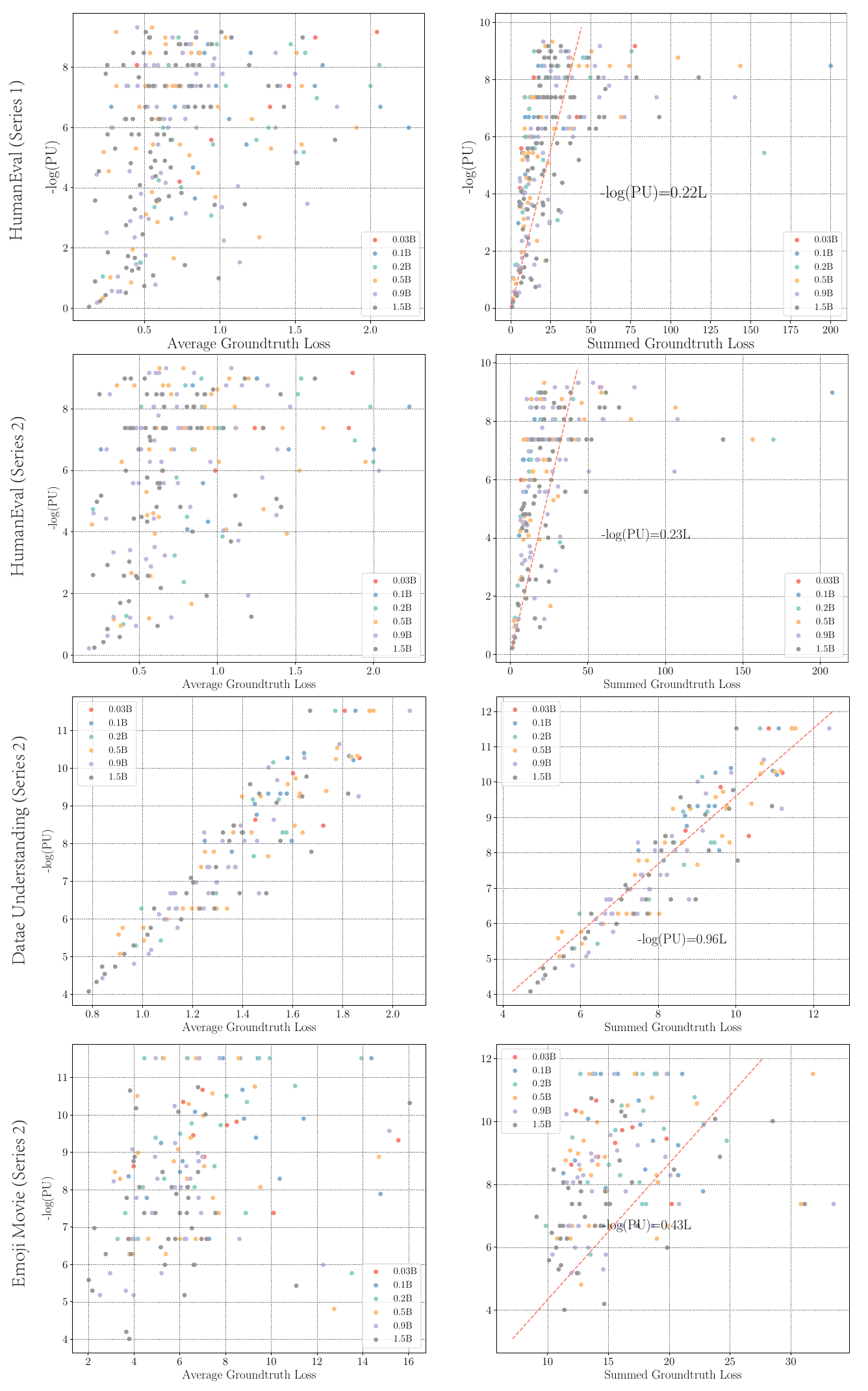}
    \caption{Additional figure on the relation between test loss and \textsc{PassUntil}.}
    \label{fig:loss_vs_passuntil}
\end{figure}


\subsection{More Results of the Unnatural In-context Learning Tasks}
\label{app:unnaturalincontextscalingcurve}
\hyperlink{back:d_3}{\faHandPointerO}
\cmt{In Figure \ref{fig:appendix_unnatural}, we present the scaling curves for the remaining fix sub-tasks of the Unnatural In-context Learning tasks. Notably, the curves in (a), (b), and (c) demonstrate a concave pattern, correlating $\log(\log(-F(N))$ with $\log N$. Specifically, the 2-digits task displays an interesting inverse scaling trend, indicating  further investigation to delineate a clearer trend.}

\cmt{
Regarding tasks in (d) and (e), we observed that these tasks pose significant challenges for smaller models. Specifically, models with 0.03B and 0.1B parameters failed to achieve non-zero pass rates, rendering the fit analysis less meaningful. Additionally, for the Reverse to Natural Content task, there's a discernible, albeit slight, sub-scaling law growth trend. This trend may be attributed to the multi-step nature inherent in this task. 
}

\begin{figure}[!htbp]
        \centering
        \includegraphics[width=0.9\linewidth]{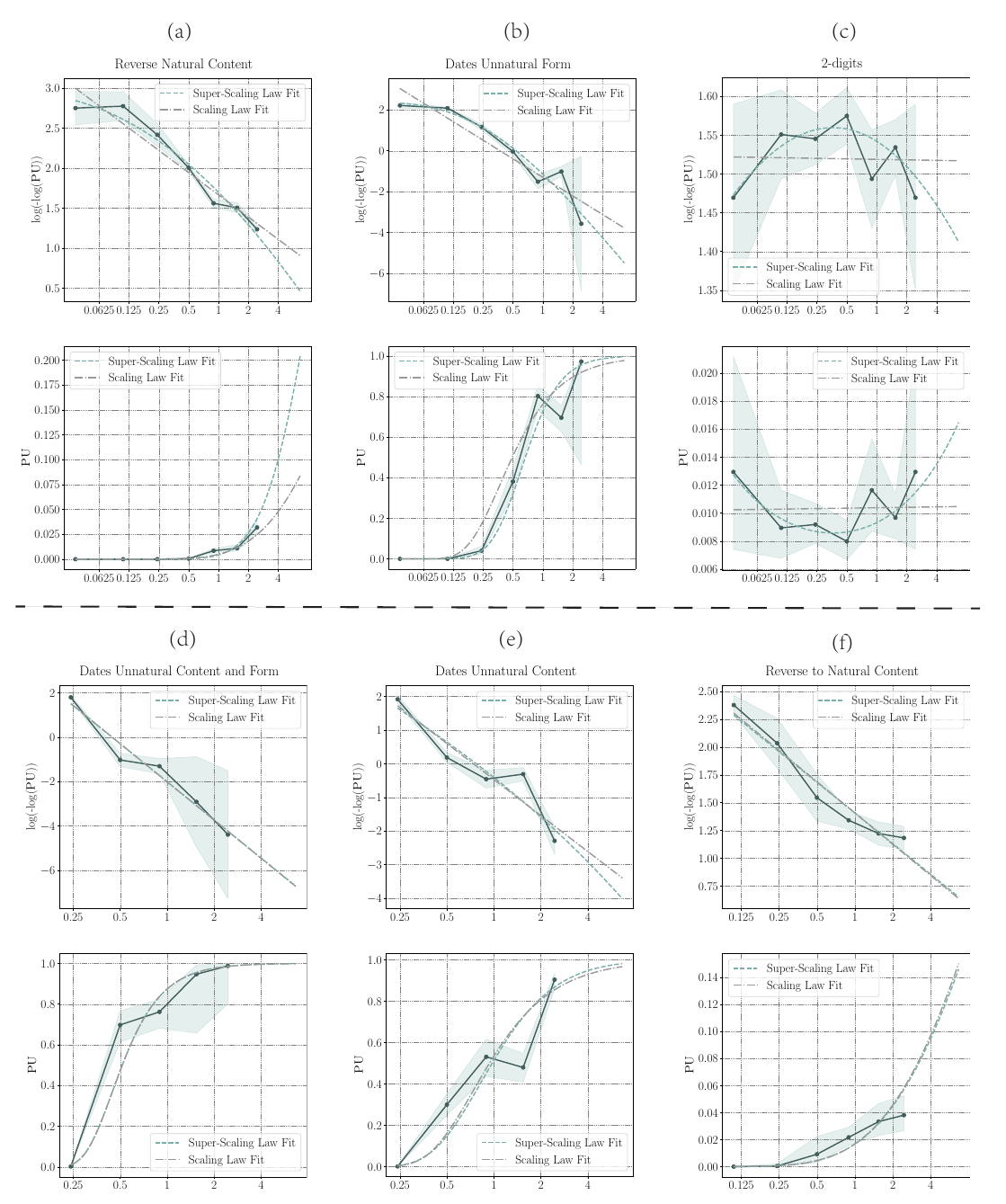}
    \caption{\cmt{Additional figure on unnatural in-context learning. The grey line shows the scaling law fit, while the green line shows the super-scaling law fit.}}
    \label{fig:appendix_unnatural}
\end{figure}

\clearpage
\subsection{Result of Individual PassUntil on More Samples }
\label{app:otheripu_sample}
\hyperlink{back:d_4}{\faHandPointerO}
Figure \ref{fig:idp_with_modelsize_part1} shows more instances of individual \textsc{PassUntil} scaling curves of model series 1 on Humaneval task.
\begin{figure}[htbp]
\centering
    \begin{minipage}[b]{0.84\textwidth}
    \centering
    \includegraphics[width=\textwidth]{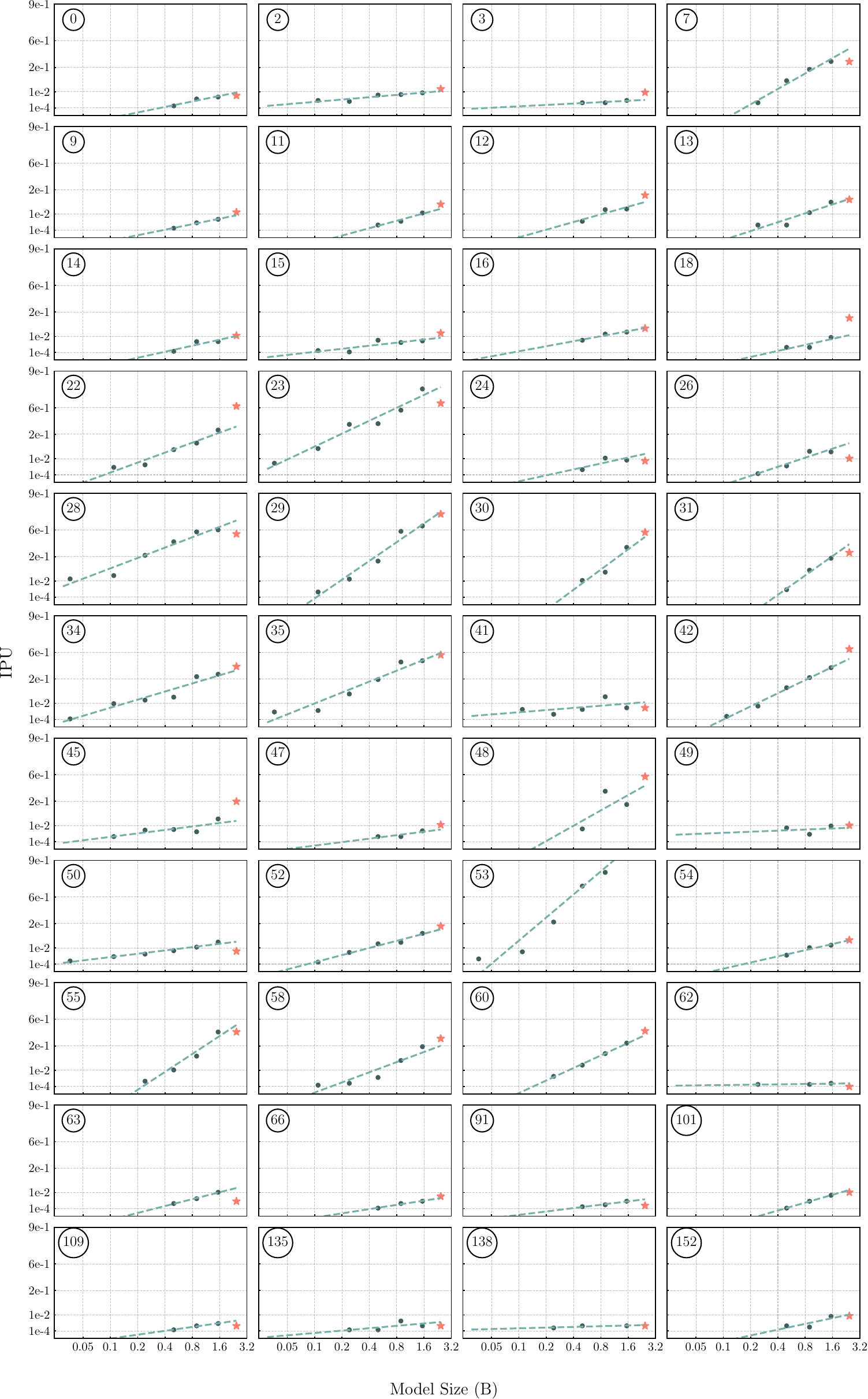}
    \vspace{-0.5cm}
    \caption{Result of instance-level scaling law fit. The label on the left upper corner of each subplot denotes the index of the sample in the test set\protect\footnotemark. }
    \label{fig:idp_with_modelsize_part1}
\end{minipage}
\hfill
\end{figure}
\footnotetext{\url{https://github.com/openai/human-eval/tree/master/data}}

\end{document}